\documentclass{article}

\usepackage[accepted]{icml2018}

\usepackage{amsfonts}
\usepackage{amsmath}
\usepackage{amssymb}
\usepackage{amsthm}
\usepackage{array}
\usepackage{booktabs}
\usepackage{color}
\usepackage{fancyhdr}
\usepackage{lineno,hyperref}
\usepackage{graphicx}
\usepackage{latexsym}
\usepackage{microtype}
\usepackage{relsize}
\usepackage{subfigure}
\usepackage{tikz}
\usetikzlibrary{plotmarks}
\usepackage{enumerate}
\usepackage{enumitem}
\setlist{
topsep=-0.1em,   
itemsep=-0.3em, 
leftmargin=1.3pc, 
 }

\newtheorem{theorem}{Theorem}
\newtheorem{result}{Main Result}
\newtheorem{corollary}{Corollary}
\newtheorem{lemma}{Lemma}
\newtheorem{definition}{Definition}

\newtheorem*{theorem*}{Theorem}
\newtheorem*{lemma*}{Lemma}
\newtheorem*{claim*}{Claim}
\newtheorem*{corollary*}{Corollary}
\newcommand{\ones}{\mathbf{1}}
\newcommand{\real}{\mathbb{R}}
\newcommand{\Loss}{\mathcal{L}}
\newcommand{\N}{\mathcal{N}}

\newcommand{\I}{\mathrm{Id}}

\DeclareMathOperator{\sign}{sign}
\newcommand{\ww}{\mathbf{w}}
\newcommand{\ee}{\mathbf{e}}

\newcommand{\vv}{\mathbf{v}}
\newcommand{\aaa}{\mathbf{a}}
\newcommand{\xx}{\mathbf{x}}
\newcommand{\sss}{\mathbf{s}}
\newcommand{\qq}{\mathbf{q}}
\newcommand{\zz}{\mathbf{z}}

\newcommand{\yy}{\mathbf{y}}
\newcommand{\bb}{\mathbf{b}}
\newcommand{\cc}{\mathbf{c}}
\newcommand{\blambda}{\boldsymbol  \lambda }
\newcommand{\bepsilon}{\boldsymbol  \varepsilon }

\newcommand{\0}{\mathbf{0}}

\newcommand{\veps}{\varepsilon}
\newcommand{\R}{\mathbb{R}}
\newcommand{\cl}{ \mathrm{cl} }
\newcommand{\bomega}{\boldsymbol \omega}
\newcommand{\x}{\xx}
\newcommand{\Sc}{\mathcal{S}}

\newcommand{\M}{\mathbb{M}}

\newcommand{\E}{\mathcal{E}}
\newcommand{\vrho}{\varrho}

\icmltitlerunning{The Multilinear Structure of ReLU Networks}

\begin{document}

\twocolumn[
\icmltitle{The Multilinear Structure of ReLU Networks}

\icmlsetsymbol{equal}{*}

\begin{icmlauthorlist}
\icmlauthor{Thomas Laurent}{equal,lmu}
\icmlauthor{James H. von Brecht}{equal,lbs}
\end{icmlauthorlist}

\icmlaffiliation{lmu}{Department of Mathematics, Loyola Marymount University, Los Angeles, CA 90045, USA}
\icmlaffiliation{lbs}{Department of Mathematics and Statistics, California State University, Long Beach, Long Beach, CA 90840, USA}

\icmlcorrespondingauthor{Thomas Laurent}{tlaurent@lmu.edu}
\icmlcorrespondingauthor{James H. von Brecht}{james.vonbrecht@csulb.edu}


\vskip 0.3in
]

\printAffiliationsAndNotice{\icmlEqualContribution}

\begin{abstract}
We study the loss surface of neural networks equipped with a hinge loss criterion and ReLU or leaky ReLU nonlinearities. Any such network defines a piecewise multilinear form in parameter space.  By appealing to harmonic analysis we show that all local minima of such network are non-differentiable, except for those minima that occur in a region of parameter space where the loss surface is perfectly flat. Non-differentiable minima are therefore not technicalities or pathologies; they are heart of the problem when investigating the loss of ReLU networks. As a consequence, we must employ techniques from nonsmooth analysis to study these loss surfaces. We show how to apply these techniques in some illustrative cases.
\end{abstract}

\section{Introduction}

Empirical practice tends to show that modern neural networks have relatively benign loss surfaces, in the sense that training a deep network proves less challenging than the non-convex and non-smooth nature of the optimization would na\"{i}vely suggest. Many theoretical efforts, especially in recent years, have attempted to explain this phenomenon and, more broadly, the successful optimization of deep networks in general 
\cite{gori1992problem, choromanska2015loss, kawaguchi2016deep, safran2016quality, mei2016landscape, soltanolkotabi2017learning, soudry2017exponentially, du2017convolutional,
zhong2017recovery, tian2017analytical, li2017convergence, zhou2017landscape, brutzkus2017sgd}.
The properties of the loss surface of neural networks remain poorly understood despite these many efforts. Developing of a coherent mathematical understanding of them is therefore one of the major open problems in deep learning.

We focus on investigating the loss surfaces that arise from feed-forward neural networks where rectified linear units (ReLUs) $\sigma(x) :=\max(x,0)$ or leaky ReLUs 
$
\sigma_\alpha(x):= \alpha \min(x,0) + \max(x,0)
$
account for all nonlinearities present in the network. 
We allow the transformations defining the hidden-layers of the network to take the form of fully connected affine transformations or convolutional transformations. By employing a ReLU-based criterion we then obtain a loss with a consistent, homogeneous structure for the nonlinearities in the network. We elect to use the binary hinge loss
\begin{equation}
\ell( \hat y ,  y)  := \sigma\big( 1 - y \hat y \big) \label{hl1}
\end{equation}
for binary classification, where $\hat y$ denote the scalar output of the network and $y \in \{-1,1\}$ denotes the target. Similarly, for multiclass classification we use
the multiclass hinge loss,
\begin{equation} \textstyle
 \ell( \hat \yy , r_0)=  \sum_{r\neq r_0} \sigma\big( 1 +  \hat y_r-   \hat y_{r_0}   \big)   \label{hl2}
\end{equation}
where $\hat \yy = (\hat y_1, \ldots, \hat y_R) \in \real^R$ denotes the vectorial output of the network  and  $r_0 \in \{1,\ldots,R\}$ denotes the target class.
 
 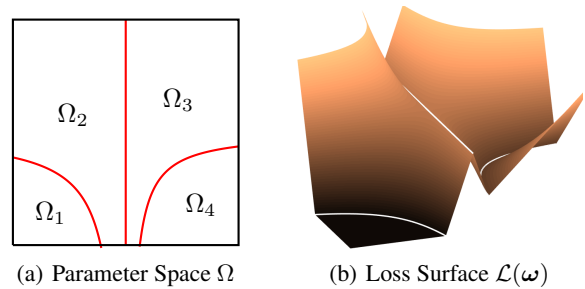
\begin{figure}[t]
\begin{center}
\subfigure[Parameter Space $\Omega$]{
\begin{tikzpicture}[scale=0.5]
\draw[red,  thick, domain=-3:-0.65, samples=500] plot (\x,    {2/ (\x)})    ;
\draw[red,  thick, -] (0,-3) -- (0, 3) ; 
\draw[red,  thick, domain=0.37:3, samples=500] plot (\x,    {- 1.125/ (\x)})    ;
\draw[ thick, -] (-3,-3) -- (3 , -3) -- (3,3) -- (-3,3) -- (-3,-3) ; 
\node [below] at (-2,-1.6) {{$\Omega_1$}};
\node [below] at (-1.4,1) {{$\Omega_2$}};
\node [below] at (1.4,1.4) {{$\Omega_3$}};
\node [below] at (2,-1.4) {{$\Omega_4$}};
\end{tikzpicture}
} \hspace{0.4cm}
\subfigure[Loss Surface $\Loss(\bomega)$]{ \includegraphics[trim=240 110 140 100,clip,width=1.5in]{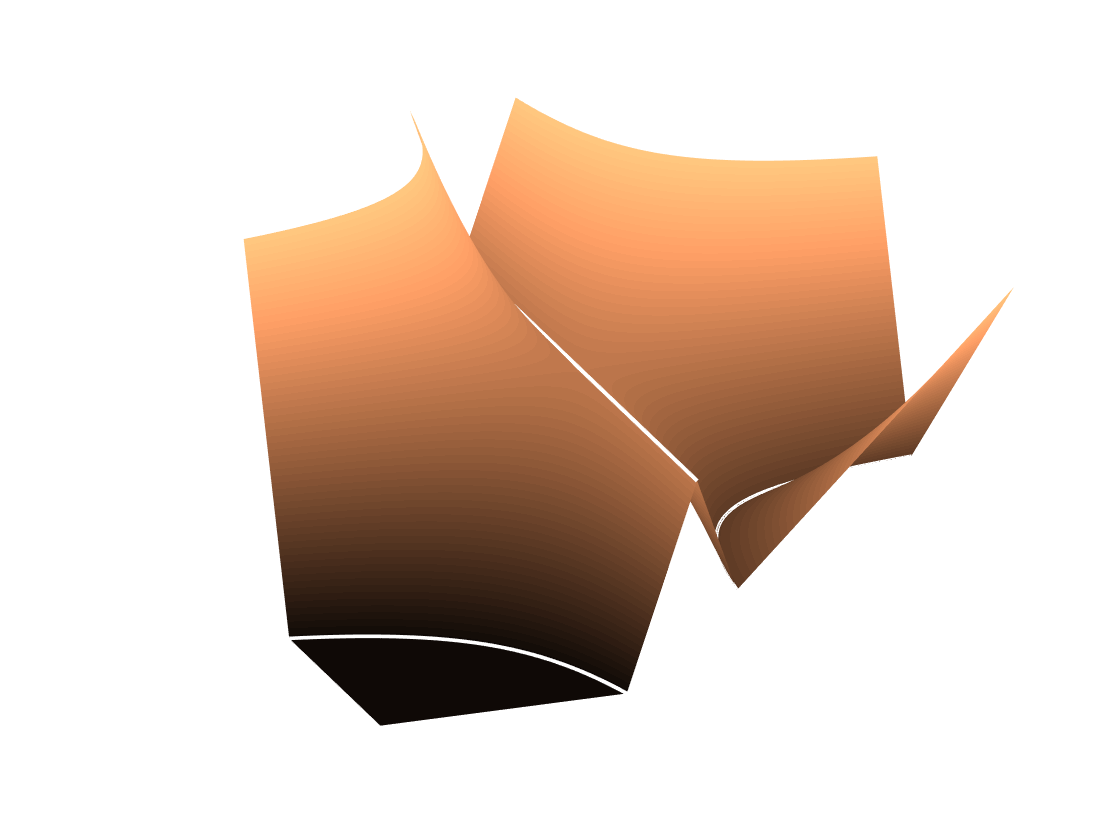} }
\caption{(a): Parameter space $\Omega = \R^{2}$ decomposes into a partition of four open cells $\Omega_{u}$ and a closed boundary set $\N$ (solid red lines). (b):  The loss surface is smooth inside each $\Omega_u$ and non-differentiable on $\N$. 
 It has two types of local minima, flat minima (cell $\Omega_1$) and sharp minima
(boundary between cells $\Omega_3$ and $\Omega_4$). The sharp minima must have non-zero loss.
 }
 \label{artificial_2d_pic}
\end{center}
\end{figure}

To see the type of structure that emerges in these networks, let $\Omega$ denote the space of network parameters and let $\Loss(\bomega)$ denote the loss. Due to the choices (\ref{hl1},\ref{hl2}) of network criteria, all nonlinearities involved in $\Loss(\bomega)$ are piecewise linear. These nonlinearities encode a partition of parameter space $ \Omega = \Omega_1 \cup \cdots \cup \Omega_{M}\cup \N$
into a finite number of open  \emph{cells} $\Omega_u$ and a closed set $\N$ of cell boundaries (c.f. figure \ref{artificial_2d_pic})\nolinebreak. A cell $\Omega_u$ corresponds to a given activation pattern of the nonlinearities, and so $\Loss(\bomega)$ is smooth in the interior of cells and (potentially) non-differentiable on cell boundaries. This decomposition provides a description of the smooth (i.e. $\Omega \backslash \N$) and non-smooth (i.e. $\N$) parts of parameter space.

  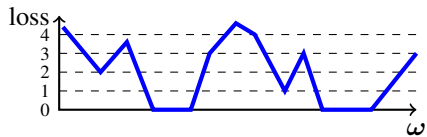
\begin{figure}[t]
\begin{center}
\begin{tikzpicture}[scale=0.5]
\draw[thick,<->] (0,2.5) -- (0,0) -- (9.5,0);
\draw[dashed,-] (0,0.5) -- (9.5,0.5);
\draw[dashed,-] (0,1) -- (9.5,1);
\draw[dashed,-] (0,1.5) -- (9.5,1.5);
\draw[dashed,-] (0,2) -- (9.5,2);
\node [below] at (9.5,-0.1) {{$\bomega$}};
\node [left] at (0,2.5) {{loss}};
\node [left] at (0,0) {\scriptsize{0}};
\node [left] at (0,0.5) {\scriptsize{1}};
\node [left] at (0,1) {\scriptsize{2}};
\node [left] at (0,1.5) {\scriptsize{3}};
\node [left] at (0,2) {\scriptsize{4}};
\draw[ultra thick, blue,-] (0.1,2.2) -- (1.1 , 1.0) --(1.8, 1.8)--(2.5, 0) -- (3.5,0) -- (4, 1.5) -- (4.7, 2.3) -- (5.2, 2)-- (6, 0.5) -- (6.5,1.5) -- (7,0) -- (8.3,0)--(9.5,1.5) ;
\end{tikzpicture}
\end{center}
\vspace{-0.5cm}
\caption{For a certain class of networks, flat local minima are always optimal whereas sharp ones are always sub-optimal.}
\label{fig:cartoon}
\end{figure}

We begin by showing that the loss restricted  to a cell $\Omega_u$ is a multilinear form. As multilinear forms are harmonic functions, an appeal to the strong maximum principle shows that non-trivial optima of the loss must happen on cell boundaries (i.e. the non-differentiable region $\N$ of the parameter space). 
In other words, ReLU networks with hinge loss criteria {\bf do not have differentiable local minima}, except for those trivial ones that occur in regions of parameter space where the loss surface is perfectly flat. Figure \ref{artificial_2d_pic}b) provides a visual example of such a loss.

As a consequence the loss function has only two types of local minima. They are
\begin{itemize}
\item {{\bf Type I (Flat)}: Local minima that occur in a flat (i.e. constant loss) cell or on the boundary  of a flat cell.} 
\item {{\bf Type II (Sharp)}: Local minima on $\N$ that  are not on the boundary  of any flat cell. }
\end{itemize}
We then investigate type I and type II local minima in more detail. The investigation reveals a clean dichotomy. First and foremost,
\begin{result} $\Loss(\bomega)>0$ at any type II local minimum.\end{result}
Importantly, if zero loss minimizers exist (which happens for most modern deep networks) then sharp local minima are always sub-optimal. This result applies to a quite general class of deep neural networks with fully connected or convolutional layers equipped with either ReLU or leaky ReLU nonlinearities. To obtain a converse we restrict our attention to  fully connected  networks with leaky ReLU nonlinearities. Under mild assumptions on the data we have
 \begin{result} $\Loss(\bomega)=0$ at any type I local minimum, while $\Loss(\bomega)>0$ at any type II local minimum.
 \end{result}
Thus flat local minima are always optimal whereas sharp minima are always sub-optimal in the case where zero loss minimizers exist. Conversely, if zero loss minimizers do not exist then all local minima are sharp. See figure \ref{fig:cartoon} for an illustration of such a loss surface.
 
 All in all these results paint a striking picture. Networks with ReLU or leaky ReLU nonlinearities and hinge loss criteria have only two types of local minima. Sharp minima always have non-zero loss; they are undesirable. Conversely, flat minima are always optimal for certain classes of networks. In this case the structure of the loss (flat v.s. sharp) provides a  perfect characterization of their quality (optimal v.s.  sub-optimal).
 
This analysis also shows that local minima generically occur in the non-smooth region of  parameter space. Analyzing them requires an invocation of machinery from non-smooth, non-convex analysis. We show how to apply these techniques to study non-smooth networks in the context of binary classification. We consider three specific scenarios to illustrate how nonlinearity and data complexity affect the loss surface of multilinear networks ---
\begin{itemize}
\item {\bf Scenario 1:} A deep linear network with arbitrary data.
\item {\bf Scenario 2:} A network with one hidden layer, leaky ReLUs and linearly separable data.
\item {\bf Scenario 3:} A network with one hidden layer, ReLUs and linearly separable data.
\end{itemize} 
The nonlinearities $\sigma_{\alpha}(x)$ vary from the linear regime $(\alpha = 1)$ to the leaky regime $(0<\alpha<1)$ and finally to the ReLU regime $(\alpha = 0)$ as we pass from the first to the third scenario. We show that no sub-optimal local minimizers exist in the first two scenarios. When passing to the case of paramount interest, i.e. the third scenario, a bifurcation occurs. Degeneracy in the nonlinearities (i.e. $\alpha=0$) induces sub-optimal local minima in the loss surface. We also provide an explicit description of  all such sub-optimal local optima. They correspond to the occurence of \emph{dead data points}, i.e. when some data points do not activate any of the neurons of the hidden layer and are therefore ignored by the network. Our results for the second and third scenarios provide a mathematically precise formulation of a commonplace intuitive picture. A ReLU can completely ``turn off," and sub-optimal minima correspond  precisely to situations in which a data point turns off all ReLUs in the hidden layer. As leaky ReLUs have no completely ``off'' state, such networks therefore have no sub-optimal minima.

Finally, in section 4 we conclude by investigating the extent to which these phenomena do, or do not, persist when passing to the multiclass context. The loss surface of a multilinear network with the multiclass hinge loss \eqref{hl2} is fundamentally different than that of a binary classification problem. In particular, the picture that emerges from our two-class results does not extend to the multiclass hinge loss. Nevertheless, we show how to obtain a similar picture of critical points by modifying the training strategy applied to multiclass problems.

Many recent works theoretically investigate the loss surface of ReLU networks. The closest to ours is
\cite{safran2016quality}, which uses ReLU nonlinearities to partition the parameter space into basins that, while similar in spirit, differ from our notion of cells. Works such as \cite{keskar2016large, chaudhari2016entropy} have empirically investigated the notion of ``width" of a local minimizer. Conjecturally, a ``wide" local minimum should generalize better than a ``narrow" one and might be more likely to attract the solution generated by a stochastic  gradient descent algorithm. Our flat and sharp local minima are reminiscent of these notions. Finally, some prior works have proved variants of our results in smooth situations. For instance, \cite{brutzkus2017sgd} derives results about the smooth local minima occurring in scenarios 2 and 3, but they do not investigate non-differentiable local minima. Additionally, \cite{kawaguchi2016deep} considers our first scenario with a mean squared error loss instead of the hinge loss, while \cite{frasconi1997successes} considers our second scenario with a smooth version of the hinge loss and with sigmoid nonlinearities. Our non-smooth analogues of these results require fundamentally different techniques. We prove all lemmas, theorems and corollaries in the appendix.

\section{Global Structure of the Loss}
We begin by describing the global structure of ReLU networks with hinge loss that arises due to their piecewise multilinear form.
Let us start by rewriting \eqref{hl2} as
\begin{align}\label{eq:lossdef}
\
\ell( \hat \yy ,  \yy) &= -1+\sum_{r=1}^R \sigma\big( 1 +  \hat y_r- \langle\yy, \hat \yy\rangle \big) \nonumber  \\
&= -1+\Big\langle \;  \ones \; , \;  \sigma \Big( \;  (\I - \ones \otimes \yy)   \hat \yy +  \ones  \;\Big) \Big\rangle
\end{align}
where we now view the target $\yy \in \{0,1\}^R$ as a one-hot vector that encodes for the desired class. The term $\ones \otimes \yy$ denotes the outer product between the constant vector $\mathbf{1} = (1,\ldots,1)^{T}$ and the target, while $\langle\yy, \hat \yy\rangle$ refers to the usual Euclidean inner product. We consider a collection  $(\xx^{(i)}, \yy^{(i)})$ of $N$ labeled data points fed through a neural network with $L$ hidden layers, 
\begin{align}
 &\xx^{(i,\ell )} \,= \sigma_\alpha( W^{(\ell)} \xx^{(i,\ell-1)}+ \bb^{(\ell)})  \qquad \text{for} \quad \ell \in [L] \nonumber \\ 
 &\hat \yy^{(i)} \;\;\,=  V \xx^{(i,L)}+ \cc,  \label{bobo} 
\end{align}
so that for $\ell \in [L]:=\{1,\ldots,L\}$ each $\xx^{(i,\ell)}$ refers to feature vector of the $i^{{\rm th}}$ data point at the $\ell^{ {\rm th}}$ layer (with the convention that $\xx^{(i,0)}=\xx^{(i)}$) and  $\hat \yy^{(i)}$ refers to the output of the network for the $i^{ {\rm th}}$ datum. By \eqref{eq:lossdef} we obtain
\begin{equation}\label{full_loss}
\Loss( \bomega ) =  -1+ \sum_i  \mu^{(i)}  \big\langle   \ones  ,  \sigma \big(   (\I - \ones \otimes \yy^{(i)})   \hat \yy^{(i)} +  \ones  \big) \big\rangle
\end{equation}
for the loss $\Loss(\bomega)$. The positive weights $\mu^{(i)} > 0$   sum to one, say $\mu^{(i)}=1/N$ in the simplest case, but we allow for other choices to handle those situations, such as an unbalanced training set, in which non-homogeneous weights could be beneficial.  
 The matrices $W^{(\ell)}$ and vector $\bb^{(\ell)}$ appearing in  \eqref{bobo}   define the affine transformation at layer $\ell$ of the network, and $V$ and $\cc$ in \eqref{bobo} denote the weights and bias of the output layer.
 We allow for fully-connected as well as structured models, such as convolutional networks, by imposing the assumption that each $W^{(\ell)}$ is a matrix-valued function that depends \emph{linearly} on some set of parameters $\omega^{(\ell)}$ ---
$$
W^{(\ell)}\big( c \omega^{(\ell)} +  d \hat \omega^{(\ell)} \big) = c W^{(\ell)}\big( \omega^{(\ell)}\big) +   d W^{(\ell)}\big( \hat \omega^{(\ell)}\big);
$$
thus the collection 
$$\bomega = (\omega^{(1)},\ldots,\omega^{(L)},V,\bb^{(1)},\ldots,\bb^{(L)},\cc) \in \Omega$$ represents the parameters of the network and $\Omega$ denotes parameter space. As the slope $\alpha$ of the nonlinearity decreases from $\alpha = 1$ to $\alpha = 0$ the network transitions from a deep linear architecture to a standard ReLU network. Finally, we let $d_\ell$ denote the dimension of the features at layer $\ell$ of the network, with the convention that $d_0=d$ (dimension of the input data) and $d_{L+1}=R$ (number of classes). We use $D=d_1+\ldots+d_{L+1}$ for the total number of neurons. 

\subsection{Partitioning $\Omega$ into Cells}
The nonlinearities $\sigma_\alpha(x)$ and $\sigma(x)$ account for the only sources of nondifferentiabilty in the loss of a ReLU network. To track these potential sources of nondifferentiability, for a given a data point $\xx^{(i)}$ we define the functions
\begin{align}\label{eq:sigs0}
\sss^{(i,\ell)}(\bomega) &:= \sign( W^{(\ell)} \xx^{(i,\ell-1)}+ \bb^\ell)   \qquad \text{for  } \quad \ell \in [L] \nonumber \\
\sss^{(i,L+1)}(\bomega) &:= \sign\left( \;  (\I - \ones\otimes \yy^{(i)})  \;  \hat\yy^{(i)}  +  \ones  \;\right),
\end{align}
where $\sign(x)$ stands for the signum function that vanishes at zero. The function $\sss^{(i,\ell)}$ describes how data point $\x^{(i)}$ activates the $d_{\ell}$ neurons at the $\ell^{{\rm th}}$ layer, while $\sss^{(i,L+1)}(\bomega)$ describes the corresponding ``activation'' of the loss. These activations take one of three possible states, the fully active state (encoded by a one), the fully inactive state (encoded by a minus one), or an in-between state (encoded by a zero). We then collect all of these functions into a single \emph{signature function}
\begin{align*}
 \Sc(\bomega)= \Big( \sss^{(1,1)}(\bomega), \ldots , \sss^{(1,L+1)}(\bomega);   \ldots \ldots; \\
  \sss^{(N,1)}(\bomega), \ldots , \sss^{(N,L+1)}(\bomega)\Big)
\end{align*}
to obtain a function $\Sc: \Omega \mapsto \{ -1  , \, 0  ,\,  1\}^{ND}$ since there are a total of  $D$  neurons and $N$  data points. If  $\Sc(\bomega)$ belongs to the subset $\{-1,1\}^{ND}$ of  $\{ -1  , \, 0  ,\,  1\}^{ND}$ then none of the $ND$ entries of $\Sc(\bomega)$ vanish, and as a consequence, all of the nonlinearities are differentiable near $\bomega$; the loss $\Loss$ is smooth near such points. With this in mind, for a given $u \in \{-1,1\}^{ND}$ we define the cell $\Omega_u$ as the (possibly empty) set
$$
\Omega_{u} := \Sc^{-1}(u) := \left\{ \bomega \in \Omega: \Sc(\bomega)=u \right\}
$$
of parameter space. By choice  $\Loss$ is smooth on each non-empty cell $\Omega_u,$ and so the cells $\Omega_u$  provide us with a partition of the parameter space
$$
\Omega=  \left(  \bigcup_{u \in \{-1,1\}^{ND} } \, \Omega_u \right) \;\;  \bigcup \;\;   \mathcal{N}
$$
into smooth and potentially non-smooth regions. The set $\mathcal{N}$ contains those $\bomega$ for which at least one of the $ND$ entries of $\Sc(\bomega)$ takes the value $0,$ which implies that at least one of the nonlinearities is non-differentiable at such a point. Thus $\mathcal{N}$ consists of points at which the loss is potentially non-differentiable. The following lemma collects the various properties of the cells $\Omega_u$ and of $\mathcal{N}$ that we will need.
\begin{lemma} \label{topology}
For each $u \in \{-1,1\}^{ND}$ the cell $\Omega_u$ is an open set. If $u \neq u'$ then $\Omega_u$ and $\Omega_{u'}$ are disjoint. The set 
$\mathcal{N}$ is  closed and has Lebesgue measure $0$.
\end{lemma}


\subsection{Flat and Sharp Minima} 
Recall that a function 
 $
 \phi: \real^{d_1} \times \ldots \times \real^{d_n} \to \real
 $ is a multilinear form if it is linear with respect to each of its inputs when the other inputs are fixed. That is, 
 \begin{align*}
\phi(\vv_1,  \ldots, c \vv_k+ d \ww_k, \ldots , \vv_n) &=  c \phi(\vv_1,  \ldots, \vv_k, \ldots , \vv_n) \\
&+ d\phi(\vv_1,  \ldots, \ww_k, \ldots , \vv_n).
\end{align*}
Our first theorem forms the basis for our analytical results. It states that, up to a constant, the loss restricted to a fixed  cell $\Omega_u$ is a sum of multilinear  forms.
\begin{theorem}[Multilinear Structure of the Loss] \label{theorem1}
\mbox{}
For each cell $\Omega_u$ there exist multilinear forms $\phi_0^u, \ldots , \phi_{L+1}^u$ and a constant $\phi_{L+2}^u$  such that
\begin{align*}
\Loss |_{\Omega_u}(\omega^{(1)}, \ldots, \omega^{(L)},V,\bb^{(1)}, \ldots,\bb^{(L)},\cc)=\\
    \phi_0^u(\omega^{(1)}, \omega^{(2)}, \omega^{(3)}, \omega^{(4)} \ldots, \omega^{(L)},V)& \\
+ \phi_1^{u}(\bb^{(1)}, \omega^{(2)}, \omega^{(3)}, \omega^{(4)} \ldots, \omega^{({L})},V)& \\
+ \phi_2^{u}(\bb^{(2)}, \omega^{(3)}, \omega^{(4)} \ldots, \omega^{({L})},V)& \\
\vdots \qquad \qquad & \\
+ \phi_{L-1}^{u}(\bb^{({L-1})},\omega^{(L)},V)&\\
+ \phi_L^{u}(\bb^{({L})},V)& \\
+ \phi_{L+1}^{u}(\cc)& \\
+ \phi_{L+2}^{u}.
\end{align*}
\end{theorem}
The proof relies on the fact that the signature function $\mathcal{S}(\bomega)$ is constant inside a fixed cell $\Omega_u,$ and so the network reduces to a succession of affine transformations. These combine to produce a sum of multilinear forms. Appealing to properties of multilinear forms then gives two important corollaries. Multilinear forms are harmonic functions. Using the strong maximum principle for harmonic functions\footnote{The strong maximum principle states that a non-constant harmonic function cannot attain a  local minimum or a local maximum at an interior point of an open, connected set.} we show that $\Loss$ does not have differentiable optima, except for the trivial flat ones.
 \begin{corollary}[No Differentiable Extrema] \label{max_principle}
 Local minima and maxima of the loss \eqref{full_loss} occur only on the boundary set $\mathcal{N}$ or on those cells $\Omega_u$ where the loss is constant. In the latter case, $\Loss |_{\Omega_u}(\bomega) = \phi_{L+2}^u$.
\end{corollary}
Our second corollary reveals the saddle-like structure of the loss.
\begin{corollary}[Saddle-like Structure of the Loss] \label{saddle_like} If $\bomega \in \Omega\setminus \mathcal{N}$ and the Hessian matrix $D^2 \Loss(\bomega)$ does not vanish, then it must have at least one strictly positive and one strictly negative eigenvalue.
\end{corollary}
These corollaries have implications for  various optimization algorithms. At a local minimum $D^{2}\Loss$ either vanishes (flat local minima) or does not exist (sharp local minima). Therefore local minima do not carry any second order information. Moreover, away from minima the Hessian is never positive definite and is typically indefinite. Thus an optimization algorithm using second-order (i.e. Hessian) information must pay close attention to both the indefinite and non-differentiable nature of the loss.

To investigate type I/II minima in greater depth we must there exploit the multilinear structure of $\Loss$ itself. Our first result along these lines concerns type II local minima.
\begin{theorem}\label{typeII}
 If $\bomega$ is a type II local minimum then $\Loss(\bomega)>0$.
\end{theorem}
 Modern networks of the form \eqref{full_loss} typically have zero loss global minimizers.  For any such network type II (i.e. sharp) local minimizers are therefore always sub-optimal. A converse of theorem \ref{typeII} holds for a restricted class of networks. That is, type I (i.e. flat) local minimizers are always optimal. To make this precise we need a mild assumption on the data.
 \begin{definition}
Fix $\alpha > 0$ and a collection of weighted data points $(\mu^{(i)},\xx^{(i)}, \yy^{(i)})$. The weighted data are {\bf rare} if there exist $N$ coeffecients $\lambda^{(i)} \in \{1,\alpha,\ldots,\alpha^{L}\}$ and a non-zero collection of $NR$ scalars $\veps^{(i,r)}\in \{0,1\}$ so that the system
\begin{align*}
\veps^{(i)} &= \sum_{ r : \yy^{(i)}_r = 0} \veps^{(i,r)} \\
\sum_{ i:\yy^{(i)}_{r} = 1 } \lambda^{(i)} \mu^{(i)}\veps^{(i)} \xx^{(i)} &= \sum_{ i:\yy^{(i)}_{r} = 0 } \lambda^{(i)} \mu^{(i)} \veps^{(i,r)}\xx^{(i)}
\end{align*}
\begin{align}\label{eq:genericdef}
\sum_{i: \yy^{(i)}_{r} = 1 } \lambda^{(i)} \mu^{(i)}\veps^{(i)} &= \sum_{i: \yy^{(i)}_{r} = 0 }\lambda^{(i)} \mu^{(i)}\veps^{(i,r)}
\end{align}
holds $\forall r \in [R]$. The data are {\bf generic} if they are not rare.
\end{definition}
As the possible choices of $\lambda^{(i)},\veps^{(i,r)}$ take on at most a finite set of values, rare data points $(\mu^{(i)},\xx^{(i)},\yy^{(i)})$ must satisfy one of a given finite set of linear combinations. Thus \eqref{eq:genericdef} represents the exceptional case,  and most data are generic. For example, if the $\xx^{(i)} \sim X^{(i)}$ come from indepenendent samples of atomless random variables $X^{(i)}$ they are generic with probability one. Similarly, a small perturbation in the weights $\mu^{(i)}$ will usually transform data from rare to generic.
\begin{theorem} \label{typeI}
Consider the loss  \eqref{full_loss} for a fully connected network. Assume that $\alpha>0$ and that the data points $(\xx^{(i)},\yy^{(i)})$ are generic. Then $\Loss(\bomega)=0$ at any type I local minimum.
\end{theorem}
For most data we may pair this result with its counterpart for fully connected networks and obtain a clear picture. Desirable (zero loss) minima are always flat, while undesireable (positive loss) minima are always sharp. Analyzing sub-optimal minima therefore requires handling the non-smooth case, and we now turn to this task.

\section{Critical Point Analysis}
In this section we use machinery from non-smooth analysis (see  chapter 6 of \cite{borwein2010convex} for a good reference) to study critical points of the loss surface of such piecewise multilinear networks. We consider three scenarios by traveling from the deep linear case ($\alpha = 1$) and passing through the leaky ReLU case ($0<\alpha<1$) before arriving at the most common case ($\alpha =0$) of ReLU networks. We intend this journey to highlight how the loss surface changes as the level of nonlinearity increases. A deep linear network has a trivial loss surface, in that local and global minima coincide (see theorem 100 in the appendix for a precise statement and its proof). If we impose further assumptions, namely linearly separable data in a one-hidden layer network, this benign structure persists into the leaky ReLU regime. When we arrive at $\alpha = 0$ a bifurcation occurs, and sub-optimal local minima suddenly appear in classical ReLU networks.

To begin, we recall that for a Lipschitz but non-differentiable function $f(\bomega)$ the \emph{Clarke subdifferential} $\partial_{0} f( \bomega)$ of $f$ at a point $\bomega \in \Omega$ provides a generalization of both the gradient $\nabla f(\bomega)$ and the usual subdifferential $\partial f(\bomega)$ of a convex function.
The Clarke subdifferential is defined as follow (c.f. page 133 of \cite{borwein2010convex}):

\begin{definition}[Clarke Subdifferential and Critical Points] 
\label{clarke}
Assume that a function $f : \Omega \mapsto \R$ is locally Lipschitz around $\bomega \in \Omega,$ and differentiable on $\Omega  \setminus \mathcal{M}$  where  $\mathcal{M}$ is a set of Lebesgue measure zero. Then the convex hull
$$
\partial_{0} f(\bomega) := \mathrm{c.h.}\left\{ \lim_{k} \nabla f(\bomega_k) : \bomega_{k} \to \bomega, \bomega_{k} \notin \mathcal{M} \right\}
$$
is the \textbf{Clarke Subdifferential} of $f$ at $\bomega$. In addition, if
\begin{equation}
\0 \in \partial_{0} f(\bomega), \label{crit}
\end{equation}
then $\bomega$ is  a \textbf{critical point} of $f$ in the Clarke sense.
\end{definition}

The definition of critical point is a consistent one, in that \eqref{crit} must hold whenever $\bomega$ is a local minimum (c.f. page 125 of \cite{borwein2010convex}). Thus the set of all critical points contains the set of all local minima. Figure \ref{clarke} provides an illustration of the Clarke Subdifferential. It depicts a  function $f: \real^2 \mapsto \real$ with global minimum at the origin, which therefore defines a critical point in the Clarke sense. While the gradient of $f(\x)$ itself does not exist at $\0$, its restrictions $f_{k} := f|_{\Omega_k}$ to the four cells $\Omega_k$ neighboring $\0$ have well-defined gradients $\nabla f_{k}(\0)$ (shown in red) at the critical point. By definition the Clarke subdifferential $\partial_{0}f (\0)$ of $f$ at $\0$ consists of all convex combinations
$$
\theta_{1} \nabla f_{1}(\0) + \theta_{2} \nabla f_{2}(\0) + \theta_{3} \nabla f_{3}(\0) + \theta_{4} \nabla f_{4}(\0)
$$
of these gradients; that some such combination vanishes (say, $\frac1{2}\nabla f_{1}(\0) + \frac1{2}\nabla f_{3}(\0) = \0$) means that $\0$ satisfies the definition of a critical point. Moreover, an element of the subdifferential $\partial_{0}f$ naturally arises from gradient descent. A gradient-based optimization path $\x^{(j+1)} = \x^{(j)} - dt^{(j)}\nabla f(\x^{(j)})$ (shown in blue) asymptotically builds, by successive accumulation at each step, a convex combination of the $\nabla f_{k}$ whose corresponding weights $\theta_k$ represent the fraction of time the optimization spends in each cell.
\begin{figure}[t]
\begin{center}
\includegraphics[width=1.3in]{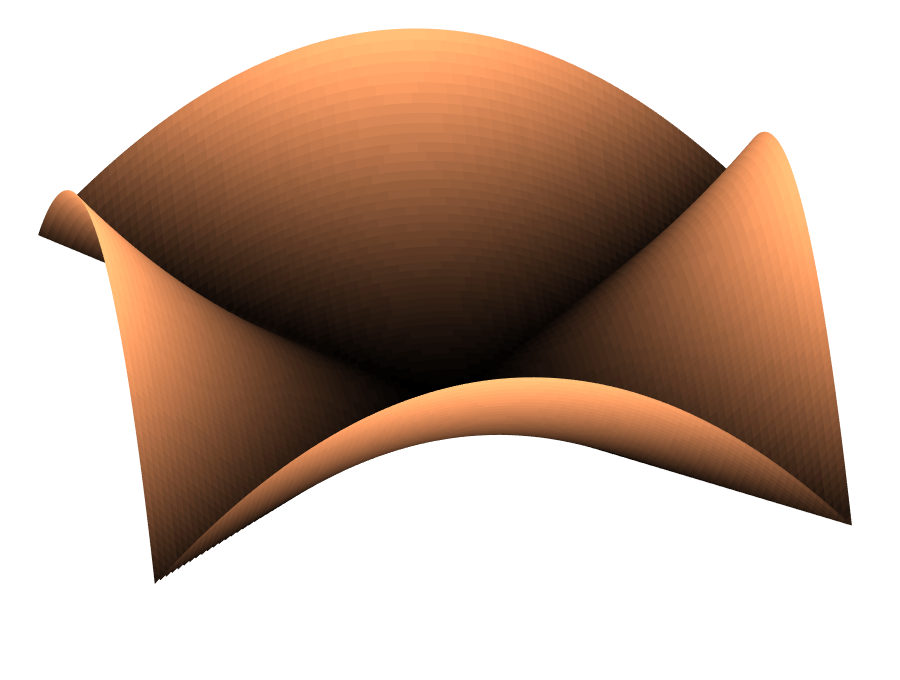}\quad\includegraphics[width=1.4in]{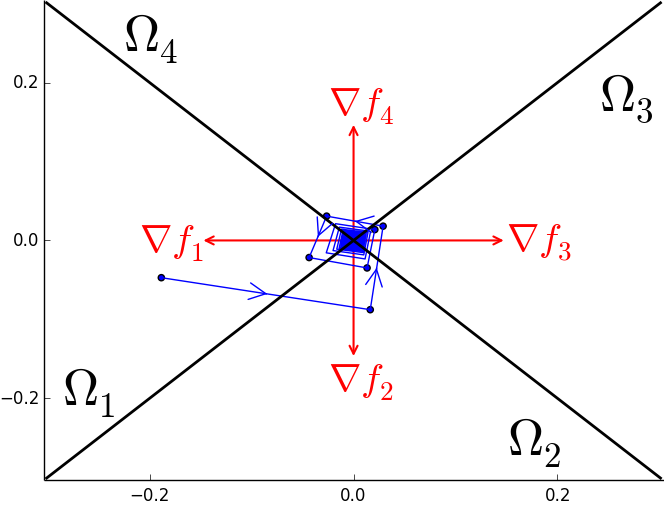}
\caption{Illustration of the Clarke Subdifferential}\label{clarke}
\end{center}
\end{figure}

We may now show how to apply these tools in the study of ReLU Networks. We first analyze the leaky regime $(0 < \alpha < 1)$ and then analyze the ordinary ReLU case $(\alpha = 0).$

\textbf{Leaky Networks ($0 < \alpha < 1$):} Take $0 < \alpha < 1$ and consider the corresponding loss $\Loss(W,\vv,\bb,c)=$
\begin{equation} \label{lll}
\sum \mu^{(i)}  \; \sigma \left[ \; 1 -y^{(i)} \left\{ \vv^T  \sigma_\alpha ( W  \xx^{(i)}+\bb)  + c     \right\}    \right]
\end{equation}
associated to a fully connected network with one hidden layer. We shall also assume the data $\{\x^{(i)}\}$ are linearly separable. In this setting we have
 \begin{theorem}[Leaky ReLU Networks] \label{crit_point_leaky_relu}
Consider the loss \eqref{lll} with  $\alpha>0$ and data $\xx^{(i)}, i \in [N] $ that are linearly separable. Assume that $\bomega=(W,\vv,\bb,c)$ is any critical point of the loss in the Clarke sense. Then either $\vv = \0$ or $\bomega$ is a global minimum.
\end{theorem}
The loss in this scenario has two type of critical points. Critical points with $\vv=\0$ correspond to a trivial network in which all data points are mapped to a constant; all other critical points are global minima. If we further assume equally weighted classes
$$
\sum_{ i : y^{(i)} = 1 } \mu^{(i)} = \sum_{ i : y^{(i)} = -1 } \mu^{(i)}
$$
then all local minima are global minima ---
\begin{theorem}[Leaky ReLU Networks with Equal Weight]\label{local_min_leaky}
Consider the loss \eqref{lll} with  $\alpha>0$ and data $\xx^{(i)}, i \in [N]$ that are linearly separable. Assume that the $\mu^{(i)}$ weight both classes equally. Then every local minimum of $\Loss(\bomega)$ is a global minimum.
\end{theorem}
In other words, the loss surface is trivial when $0 < \alpha \leq 1$.

\textbf{ReLU Networks ($\alpha=0$):} This is the case of paramount interest. When passing from $\alpha > 0$ to $\alpha = 0$ a structural bifurcation occurs in the loss surface --- ReLU nonlinearities generate non-optimal local minima even in a one hidden layer network with separable data. Our analysis provides an explicit description of all the critical points of such loss surfaces, which allows us to precisely understand the way in which sub-optimality occurs.

In order to describe this structure let us briefly assume that we have a simplified model with two hidden neurons,  no output bias and uniform weights. If $\ww_k$ denotes the $k^{{\rm th}}$ row of  $W$ then we have the loss
\begin{align} \label{bbb}
\Loss(W,\vv,\bb)&=\frac{1}{N}\sum    \sigma ( 1 -y^{(i)}  \hat y^{(i)}   ) , \quad \text{where} \nonumber \\ 
 \hat y^{(i)} &= \sum_{k=1}^2 v_k \sigma \Big( \langle \ww_k, \xx^{(i)} \rangle + b_k \Big) 
\end{align}
for such a network. Each hidden neuron has an associated hyperplane $\langle \ww_k, \cdot \rangle + b_k$ as well as a scalar weight $v_k$ used to form the output. Figure \ref{local_min_pic} shows three different local minima of  such a network. The first panel, figure \ref{local_min_pic}(a), shows a global minimum where all the data points have zero loss. Figure \ref{local_min_pic}(b) shows a sub-optimal local minimum. All unsolved data points, namely those that contribute a non-zero value to the loss, lie on the ``blind side" of the two hyperplanes. For each of these data points the corresponding network output $\hat y^{(i)}$ vanishes and so the loss is $\sigma ( \; 1 -y^{(i)}  \hat y^{(i)}   )=1$ for these unsolved points. Small perturbations of  the hyperplanes or of the values of the $v_k$ do not change the fact that these data points lie on the blind side of the two hyperplanes. Their loss will not decrease under small perturbations, and so the configuration is, in fact, a local minimum. The same reasoning shows that the configuration in figure \ref{local_min_pic}(c), in which no data point is classified correctly, is also a local minimum.


  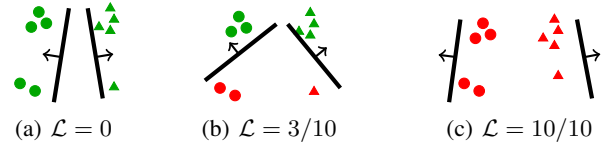
\begin{figure}[t]
\begin{center}
\subfigure[ $\Loss=0$]{
\begin{tikzpicture}[scale=0.5]
\node at (-0.5,-0.2) [black!35!green] {\pgfuseplotmark{*}};
\node at (-0.1,-0.4) [black!35!green] {\pgfuseplotmark{*}};

\node at (-0.2,1.3) [black!35!green] {\pgfuseplotmark{*}};
\node at (0.2,1.4) [black!35!green] {\pgfuseplotmark{*}};
\node at (0,1.7) [black!35!green] {\pgfuseplotmark{*}};

\node at (1.6,1.3) [black!35!green] {\pgfuseplotmark{triangle*}};
\node at (2,1.5) [black!35!green] {\pgfuseplotmark{triangle*}};
\node at (1.9,1.2) [black!35!green] {\pgfuseplotmark{triangle*}};
\node at (1.9,1.8) [black!35!green] {\pgfuseplotmark{triangle*}};

\node at (2,-0.3) [black!35!green] {\pgfuseplotmark{triangle*}};

\draw[ultra thick, -] (0.4,-0.7) -- (0.8 , 1.8); 
\draw[thick, ->] (0.6,0.5) -- (0.6-2.5/5,0.5+0.4/5);

\draw[ultra thick, -] (1.7,-0.6) -- (1.3 , 1.8); 
\draw[thick, ->] (1.5,0.5) -- (1.5+2.4/5,0.5+0.4/5);

\end{tikzpicture}} \hspace{0.8cm}
\subfigure[$\Loss=3/10$]{
\begin{tikzpicture}[scale=0.5]
\node at (-0.5,-0.2) [red] {\pgfuseplotmark{*}};
\node at (-0.1,-0.4) [red] {\pgfuseplotmark{*}};

\node at (-0.2,1.3) [black!35!green] {\pgfuseplotmark{*}};
\node at (0.2,1.4) [black!35!green] {\pgfuseplotmark{*}};
\node at (0,1.7) [black!35!green] {\pgfuseplotmark{*}};

\node at (1.6,1.3) [black!35!green] {\pgfuseplotmark{triangle*}};
\node at (2,1.5) [black!35!green] {\pgfuseplotmark{triangle*}};
\node at (1.9,1.2) [black!35!green] {\pgfuseplotmark{triangle*}};
\node at (1.9,1.8) [black!35!green] {\pgfuseplotmark{triangle*}};

\node at (2,-0.3) [red] {\pgfuseplotmark{triangle*}};

\draw[thick, ->] (0,0.7) -- (0-1.5/6,0.7+1.9/6);
\draw[ultra thick, -] (-0.9,0) -- (1 , 1.5);

\draw[ultra thick, -] (1.3,1.5) -- (2.7 , -0.2); 
\draw[thick, ->] (2,0.6) -- (2+1.7/5,0.6+1.4/5);

\end{tikzpicture}} \hspace{1cm}
\subfigure[$\Loss=10/10$]{
\begin{tikzpicture}[scale=0.5]
\node at (-0.5,-0.2) [red] {\pgfuseplotmark{*}};
\node at (-0.1,-0.4) [red] {\pgfuseplotmark{*}};

\node at (-0.2,1) [red] {\pgfuseplotmark{*}};
\node at (0.2,1.1) [red] {\pgfuseplotmark{*}};
\node at (0,1.4) [red] {\pgfuseplotmark{*}};

\node at (1.6,1) [red] {\pgfuseplotmark{triangle*}};
\node at (2,1.2) [red] {\pgfuseplotmark{triangle*}};
\node at (1.9,0.8) [red] {\pgfuseplotmark{triangle*}};
\node at (1.9,1.5) [red] {\pgfuseplotmark{triangle*}};

\node at (2,0) [red] {\pgfuseplotmark{triangle*}};

\draw[ultra thick, -] (0.4-1.3,-0.7) -- (0.7-1.3 , 1.5); 
\draw[ultra thick, -] (1.6+1.3,-0.6) -- (1.2+1.3 , 1.5); 

\draw[thick, ->] (0.6-1.3,0.5) -- (0.6-1.3-2.4/5,0.5+0.3/5);
\draw[thick, ->] (1.4+1.3,0.5) -- (1.4+1.3+2.1/5,0.5+0.4/5);

\end{tikzpicture} }

\end{center}
\vspace{-0.5cm}
\caption{Three different local minima of the loss $\Loss(\bomega)$ for a network with two hidden neurons and standard ReLU nonlinearities. Points belonging to class +1 (resp. -1) are denoted by  circles (resp. triangles). Data points for which the loss is zero (solved points) are colored in green, while data points with non-zero loss (unsolved points) are in red. The unsolved data points always lie on the blind side of both hyperplanes.}
 \label{local_min_pic}
\end{figure}

Despite the presence of sub-optimal local minimizers, the local minima depicted in figure \ref{local_min_pic} are somehow trivial cases. They simply come from the fact that, due to inactive ReLUs, some data points are completely  ignored by the network, and this fact cannot be changed by small perturbations. The next theorem essentially shows that these are the only possible sub-optimal local minima that occur. Moreover, the result holds for the case \eqref{lll} of interest and not just the simplified model.
\begin{theorem}[ReLU networks]  \label{crit_point_plain}
Consider the loss \eqref{lll} with  $\alpha=0$ and data $\xx^{(i)}, i \in [N] $ that are linearly separable. Assume that $\bomega=(W,\vv,\bb,c)$ is a critical point in the Clarke sense, and that $\xx^{(i)}$ is any data point that contributes a nonzero value to the loss. Then for each hidden neuron $k \in [K]$ either
$$
\text{(i)} \; \langle \ww_k, \xx^{(i)} \rangle + b_k \le 0,\quad \text{or}\quad \text{(ii)}\;v_k=0.
$$
\end{theorem}
If $v_k=0$ then the $k^{{\rm th}}$ hidden neuron is unused when forming network predictions. In this case we may say the $k^{{\rm th}}$  hyperplane is {\it inactive}, while if $v_k \neq 0$ the corresponding hyperplane is \emph{active}. Theorem \ref{crit_point_plain} therefore states that {\bf if a data point $\xx^{(i)}$ is unsolved it must lie on the blind side of every active hyperplane.} So all critical points, including local minima, obey the property sketched in  figure \ref{local_min_pic}.

When taken together, theorems \ref{local_min_leaky} and \ref{crit_point_plain} provide  rigorous mathematical ground for the common view that dead or inactive neurons can cause difficulties in optimizing neural networks, and that using leaky ReLU networks can overcome these difficulties. The former have sub-optimal local minimizers exactly when a data point does not activate any of the ReLUs in the hidden layer, but this situation never occurs with leaky ReLUs and so neither do sub-optima minima.

%
%

%
%
%
%
%

\section{Exact Penalties and Multi-Class Structure}

These results give a clear illustration of how nonlinearity and data complexity combine to produce local minimizers in the loss surface for binary classification tasks. While we might try to analyze multi-class tasks by following down the same path, such an effort would unfotunately bring us to a quite different destination. Specifically, the conclusion of theorem \ref{crit_point_plain} fails for multi-class case; in the presence of three or more classes a critical point may exhibit active yet unsolved data points (c.f. figure \ref{fig:multibad}). This phenomenon is inherent to multi-class tasks in a certain sense, for if we use the same features $\xx^{(i,\ell)}$ (c.f. \eqref{bobo}) in a multi-layer ReLU network but apply a different network  criterion $\bar \ell(\yy,\hat \yy)$ then the phenomenon persists. For example, using the one-versus-all criterion
\begin{equation}\label{eq:ova} \textstyle
\bar \ell(\hat \yy, \yy ) := \sum_r \mu^{(i,r)}\sigma \left( 1 + \hat y^{(i)}_{r}(-1)^{ y^{(i)}_{r} }\right),
\end{equation}
in place of the hinge loss \eqref{hl2} still gives rise to a network with non-trivial critical points (similar to figure \ref{fig:multibad}) despite its more ``binary'' structure. In this way, the emergence of non-trivial critical points reflects the nature of multi-class tasks rather than some pathology of the hinge-loss network criterion itself.
\begin{figure}[t]
\begin{centering}
\subfigure[ $\Loss(\bomega)=0$]{ \includegraphics[trim={4cm 2cm 4cm 2cm},clip,width=1in]{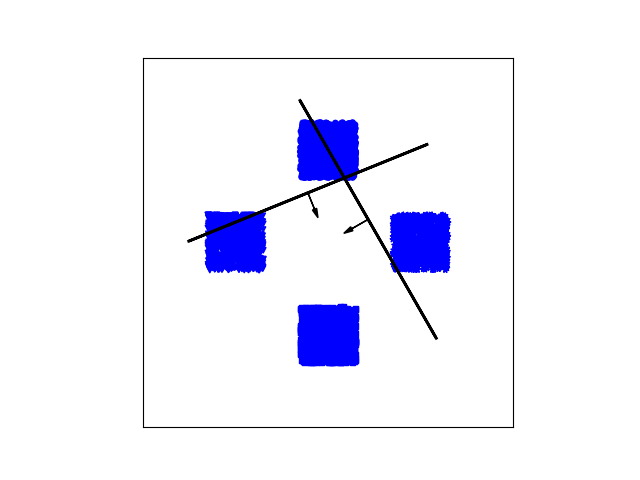} } $\qquad $
\subfigure[$\Loss(\bomega) > 0 $]{\includegraphics[trim={4cm 2cm 4cm 2cm},clip,width=1in]{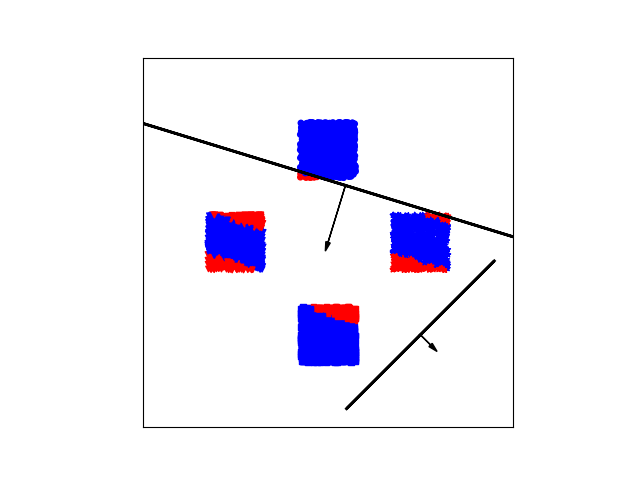}}
\caption{Four-way classification with multiclass hinge loss \eqref{hl2}. At left --- a global minimizer. At right --- a sub-optimal local minimizer where the analogue of theorem \ref{crit_point_plain} fails.}
 \label{fig:multibad}
\end{centering}
\end{figure}

To arrive at the same destination our analysis must therefore take a more circumlocuitous route. As these counter-examples suggest, if the loss $\Loss(\bomega)$ has non-trivial critical points then we must avoid non-trivial critical points by modifing the \emph{training strategy} instead. We shall employ the one-versus-all criterion \eqref{eq:ova} for this task, as this choice will allows us to directly leverage our binary analyses. 

Let us begin this process by recalling that
$$
\xx^{(i,L)}\big( \omega^{(1)},\ldots,\omega^{(L)} , \bb^{(1)},\ldots,\bb^{(L)} \big)
$$
and $\hat \yy^{(i)} = V\xx^{(i,L)} + \cc$ denote the features and predictions of the network with $L$ hidden layers, respectively. The sub-collection of parameters
$$
\breve \bomega := \big( \omega^{(1)},\ldots,\omega^{(L)} , \bb^{(1)},\ldots,\bb^{(L)} \big)
$$
therefore determine a set of features $\xx^{(i,L)}$ for the network while the parameters $V,\cc$ determine a set of one-versus-all classifiers utilizing these features. We may write the loss for the $r^{ {\rm th}}$ class as
\begin{equation}\label{eq:coupledloss}
\Loss^{(r)}( \breve \bomega , \vv_{r} , c_r ) = \sum \mu^{(i,r)}\sigma\left( 1 + \hat y^{(i)}_{r}(-1)^{ y^{(i)}_{r} }\right)
\end{equation}
and then form the sum over classes
$$\bar \Loss(\bomega) := (\Loss^{(1)} + \cdots + \Loss^{(R)})(\bomega)$$ to recover the total objective. We then seek to minimize $\bar \Loss$ by applying a soft-penalty approach. We introduce the $R$ replicates 
$$\breve \bomega^{(r)} = \big(\omega^{(1,r)},\ldots,\omega^{(L,r)},\bb^{(1,r)},\ldots,\bb^{(L,r)}\big) \quad r \in [R]$$
of the hidden-layer parameters $\breve \bomega$ and include a soft $\ell^{2}$-penalty $\mathcal{R}\big(\breve \bomega^{(1)},\ldots,\breve \bomega^{(R)}\big) := $
$$
\frac{R}{R-1} \sum^{L}_{\ell = 1} \sum^{R}_{r=1} \|\omega^{(\ell,r)} - \bar \omega ^{(\ell)} \|^{2} + \| \bb^{(\ell,r)} - \bar \bb ^{(\ell)}\|^{2}
$$
to enforce that the replicated parameters $\omega^{(\ell,r)},\bb^{(\ell,r)}$ remain close to their corresponding means $(\bar{\omega}^{(\ell)},\bar \bb^{(\ell)})$ across classes. Our training strategy then proceeds to minimize the penalized loss $\E_{\gamma}\big(\bomega^{(1)},\ldots,\bomega^{(R)}\big) := $
\begin{equation}\label{eq:pen}
\textstyle \sum_r \Loss^{(r)}\big(\bomega^{(r)}\big) + \gamma \mathcal{R}\big(\breve \bomega^{(1)},\ldots,\breve \bomega^{(R)}\big)
\end{equation}
for $\gamma > 0$ some parameter controlling the strength of the penalty. Remarkably, utilizing this strategy yields
\begin{theorem}[Exact Penalty and Recovery of Two-Class Structure]\label{thm:multi}
If $\gamma > 0$ then the following hold for \eqref{eq:pen} ---
\begin{enumerate}[label=(\roman*)]
\item The penalty is exact, that is, at \textbf{any} critical point $\big( \bomega^{(1)},\ldots,\bomega^{(R)} \big)$ of $\E_{\gamma}$ the equalities
\begin{align*}
\omega^{(\ell,1)} &= \cdots = \omega^{(\ell,R)} = \bar \omega^{(\ell)} := \frac1{R} \sum^{R}_{r=1} \omega^{(\ell,r)}\\
\bb^{(\ell,1)} &= \cdots = \bb^{(\ell,R)} = \bar \bb^{(\ell)} := \frac1{R} \sum^{R}_{r=1} \bb^{(\ell,r)}
\end{align*}
hold for all $\ell \in [L]$.
\item At \textbf{any} critical point of $\E_{\gamma}$ the two-class critical point relations $\0 \in \partial_{0}\Loss^{(r)}( \breve \bomega , \vv_r , c_r )$
hold for all $r \in [R]$.
\end{enumerate}
\end{theorem}
In other words, applying a soft-penalty approach to minimizing the original problem \eqref{eq:coupledloss} actually yields an exact penalty method. By (i), at critical points we obtain a common set of features $\xx^{(i,L)}$ for each of the $R$ binary classification problems. Moreover, by (ii) these features simultaneously yield critical points 
\begin{equation}\label{eq:magic}
\0 \in \partial_{0} \Loss^{(r)}\big(\breve \bomega , \vv_r , c_r \big)
\end{equation}
for \emph{all} of these binary classification problems. The fact that \eqref{eq:magic} may fail for critical points of $\bar \Loss$ is responsible for the presence of non-trivial critical points in the context of a network with one hidden layer. We may therefore interpret (ii) as saying that a training strategy that uses the penalty approach will avoid pathological critical points where $\0 \in \partial_{0} \bar \Loss(\bomega)$ holds but \eqref{eq:magic} does not. In this way the penalty approach provides a path forward for studying multi-class problems. Regardless of the number $L$ of hidden layers, it allows us to form an understanding of the family of critical points \eqref{eq:magic} by reducing to a study of critical points of binary classification problems. This allows us to extend the analyses of the previous section to the multi-class context. 

We may now pursue an analysis of multi-class problems by traveling along the same path that we followed for binary classification. That is, a deep linear network ($\alpha = 1$) once again has a trivial loss surface (see corollaries 100 and 101 in the appendix for precise statements and proofs). By imposing the same further assumptions, namely linearly separable data in a one-hidden layer network, we may extend this benign structure into the leaky ReLU regime. Finally, when $\alpha = 0$ sub-optimal local minima appear; we may characterize them in a manner analogous to the binary case.

To be precise, recall the loss
\begin{align}\label{eq:lrmulti}
\Loss(\bomega) &= \sum^{R}_{r=1} \Loss^{(r)}(\bomega \big) \quad \text{for} \\
 \Loss^{(r)}(\bomega) &:= \sum\mu^{(i,r)} \sigma\big( 1 - y^{(i,r)}( \langle \vv_r , \xx^{(i,1)} \rangle + c_r) \big) \nonumber
\end{align}
that results from the features $\xx^{(i,1)} = \sigma_{\alpha}(W\xx^{(i)} + \bb)$ of a ReLU network with one hidden layer. If the positive weights $\mu^{(i,r)}>0$ satisfy
$$
\sum_{y^{(i,r)}=1} \mu^{(i,r)}  = \sum_{ y^{(i,r)}=-1} \mu^{(i,r)} = \frac1{2}
$$
then we say that the $\mu^{(i,r)}$ give equal weight to all classes. Appealing to the critical point relations \eqref{eq:magic} yields the following corollary. It gives the precise structure that emerges from the leaky regime $0 < \alpha < 1$ with separable data ---
\begin{corollary}[Multiclass with $0 < \alpha < 1$]
Consider the loss  \eqref{eq:lrmulti} and its corresponding penalty \eqref{eq:pen} with $\gamma >0, 0 < \alpha < 1$ and data $\x^{(i)}, i \in [N]$ that are linearly separable. 
\begin{enumerate}[label=(\roman*)]
\item Assume that $\bomega = (\bomega^{(1)},\ldots,\bomega^{(R)})$ is a critical point of $\E_{\gamma}$ in the Clarke sense. If $\vv^{(r)} \neq \0$ for all $r \in [R]$ then $\bomega$ is a global minimum of $\Loss$ and of $\E_{\gamma}$.
\item Assume that the $\mu^{(i,r)}$ give equal weight to all classes. If $\bomega = (\bomega^{(1)},\ldots,\bomega^{(R)})$ is a local minimum of $\E_{\gamma}$ and $\vv_r = \0$ for some $r \in [R]$ then $\bomega$ is a global minimum of $\Loss$ and of $\E_{\gamma}$.
\end{enumerate}
\end{corollary}
Finally, when arriving at the standard ReLU nonlinearity $\alpha=0$ a bifurcation occurs. Sub-optimal local minimizers of $\E_{\gamma}$ can exist, but once again the manner in which these sub-optimal solutions appear is easy to describe. We let $\ell^{(i,r)}(\bomega)$ denote the contribution of the $i^{ {\rm th}}$ data point $\xx^{(i)}$ to the loss $\Loss^{(r)}$ for the $r^{ {\rm th}}$ class, so that $\Loss^{(r)}(\bomega)=  \sum_i  \mu^{(i,r)} \ell^{(i,r)}(\bomega)$ gives the total loss. Appealing directly to the family of critical point relations $\0 \in \partial_{0} \Loss^{(r)}\big(\breve \bomega , \vv_r , c_r \big)$
furnished by theorem \ref{thm:multi} yields our final corollary in the multiclass setting.
\begin{corollary}[Multiclass with $\alpha = 0$]
Consider the loss  \eqref{eq:lrmulti} and its corresponding penalty \eqref{eq:pen} with $\gamma >0, \alpha = 0 $ and data $\x^{(i)}, i \in [N]$ that are linearly separable. Assume that $\bomega=(\bomega^{(1)},\ldots,\bomega^{(R)})$ is any critical point of $\E_{\gamma}$ in the Clarke sense. Then $\ell^{(i,r)} > 0 $ 
 $$ \Longrightarrow  \quad  (\vv_r)_k \;  \sigma( \langle \ww_k, \xx^{(i)} \rangle + b_k)=0 \quad \text{for all} \quad k\in [K].$$
\end{corollary}

\section{Conclusion}
We conclude by painting the overall picture that emerges from our analyses. The loss of a ReLU network is a multilinear form inside each cell. Multilinear forms are harmonic functions, and so maxima or minima simply  cannot occur in the interior of a cell unless the loss is constant on the entire cell.  This simple harmonic analysis reasoning leads to the following striking fact. ReLU networks {\bf do not have differentiable minima}, except for trivial cases. This reasoning is valid for any convolutional or fully connected network, with plain or leaky ReLUs, and with binary or multiclass hinge loss. Dealing with non-differentiable minima is therefore not a technicality; it is the heart of the matter.

Given this dichotomy between trivial, differentiable minima on one hand and nontrivial, nondifferentiable minima on the other, it is natural to try and characterise these two classes of minima more precisely. We show that global minima with zero loss must be trivial, while minima with nonzero loss are necessarily nondifferentiable for many fully connected networks. In particular, if a network has no zero loss minimizers then  all minima are nondifferentiable. 
 
Finally, our analysis clearly shows that local minima of ReLU networks are generically nondifferentiable. They cannot be waved away as a technicality, so any study of the loss surface of such network  must invoke nonsmooth analysis. We show how to properly use this machinery (e.g. Clark subdifferentials) to study ReLU networks. Our goal is twofold. First, we prove that a bifurcation occurs when passing from leaky ReLU to ReLU nonlinearities, as suboptimal minima suddenly appear in the latter case. Secondly, and perhaps more importantly, we show how to apply nonsmooth analysis in familiar settings so that future researchers can adapt and extend our techniques.
\bibliography{bibliography}
\bibliographystyle{icml2018}

\newpage

\onecolumn
{\color{white}.}
\vskip 0.3 in

\begin{center}
{\Large \bf  Appendix: Proofs of Lemmas and Theorems}
\end{center}

\vskip 0.4in

\setcounter{equation}{99}
\setcounter{lemma}{99}
\setcounter{theorem}{99}
\setcounter{corollary}{99}

All numbering in the appendix (corollaries, lemmas, theorems and equations) begin at $100$ to distinguish them from their counterparts in the main text; any equation number or theorem number below $100$ refers to a theorem or equation in the main text.

\section*{Proofs of  lemmas and theorems from  section 2: Global Structure of the Loss}

\begin{lemma*}[Lemma 1 from the paper] \label{topologyappx}
For each $u \in \{-1,1\}^{ND}$ the cell $\Omega_u$ is an open set. If $u \neq u'$ then $\Omega_u$ and $\Omega_{u'}$ are disjoint. The set 
$\mathcal{N}$ is  closed and has Lebesgue measure $0$.
\end{lemma*}
\begin{proof}
The features $\x^{(i,\ell)}$ at each hidden layer depend in a Lipschitz fashion on parameters. Thus each $\Omega_{u}$ defines an open set in parameter space. Moreover, if $u \neq \tilde u$ then $\Omega_u$ and $\Omega_{\tilde u}$ are disjoint by definition. That $\mathcal{N}$ is closed follows from the fact that it is the complement of an open set. If $\bomega \notin \Omega_{u}$ for all $u \in \{-1,1\}^{ND}$ then at least one of the equalities
\begin{equation}\label{eq:lipgraphs}
\bb^{(\ell)}_{j} = - \langle \ww^{(\ell)}_{j} , \x^{(i,\ell-1)} \rangle \quad \text{or} \quad \cc_{s} =  -\big( 1 + \langle \vv_{s} - \vv_{r}  , \x^{(i,L)} \rangle \big),
\end{equation}
must hold. In the above equation $\bb^{(\ell)}_{j}$ and $\cc_j$ stands for the $j^{th}$ entry of the bias vectors $\bb^{(\ell)}$ and $\cc$ appearing in equation (4), whereas  $\ww^{(\ell)}_{j}$ and $\vv_j$ stands for the $j^{th}$ row of the weight matrices $W^{(\ell)}$ and $V$.
 The set of parameters $\mathcal{N} \subset \Omega$ where an equality of the form \eqref{eq:lipgraphs} holds corresponds to a Lipschitz graph in $\Omega$ of the bias parameter for that equality. It therefore follows that
$$
\mathcal{N} := \Omega \setminus \left(  \bigcup_{u \in \{-1,1\}^{ND} } \, \Omega_u \right)
$$
defines a set contained in a finite union of Lipschitz graphs. 
Thus $\mathcal{N}$ is $(n_p-1)$-rectifiable,
where  $n_{p} := \mathrm{dim}(\Omega)$ denotes  the total number of parameters. 
This implies that $\mathcal{N}$  has Lebesgue measure zero.
\end{proof}

\begin{theorem*}[Theorem 1 from the paper] \label{theorem1appx}
\mbox{}
For each cell $\Omega_u$ there exist multilinear forms $\phi_0^u, \ldots , \phi_{L+1}^u$ and a constant $\phi_{L+2}^u$  such that
\begin{align*}
\Loss |_{\Omega_u}(\omega^{(1)}, \ldots, \omega^{(L)},V,\bb^{(1)}, \ldots,\bb^{(L)},\cc)=\\
    \phi_0^u(\omega^{(1)}, \omega^{(2)}, \omega^{(3)}, \omega^{(4)} \ldots, \omega^{(L)},V)& \\
+ \phi_1^{u}(\bb^{(1)}, \omega^{(2)}, \omega^{(3)}, \omega^{(4)} \ldots, \omega^{({L})},V)& \\
+ \phi_2^{u}(\bb^{(2)}, \omega^{(3)}, \omega^{(4)} \ldots, \omega^{({L})},V)& \\
\vdots \qquad \qquad & \\
+ \phi_{L-1}^{u}(\bb^{({L-1})},\omega^{(L)},V)&\\
+ \phi_L^{u}(\bb^{({L})},V)& \\
+ \phi_{L+1}^{u}(\cc)& \\
+ \phi_{L+2}^{u}.
\end{align*}
\end{theorem*}
\begin{proof}
Let us define the collection of functions
\begin{align}\label{eq:sigs1}
\blambda^{(i,\ell)}(\bomega) &:= \sigma_\alpha'( W^{(\ell)} \xx^{(i,\ell-1)}+ \bb^\ell)   \qquad \text{for  } \quad \ell \in [L] \nonumber \\
\bepsilon^{(i)}(\bomega) &:= \sigma' \left( \;  (\I - \ones\otimes \yy^{(i)})  \;  \hat\yy^{(i)}  +  \ones  \;\right),
\end{align}
Comparing these equations with the definition (6) of the signature functions $\sss^{(i,\ell)}(\bomega)$ it is obvious
that $\blambda^{(i,\ell)}(\bomega)$ and $\bepsilon^{(i)}(\bomega)$ remain constant on each cell $\Omega_u$. We may therefore refer unambiguously to these functions by referencing a given cell $\Omega_u$ instead of a point $\bomega$ in parameter space. We shall therefore interchangably use the more convenient notation
 \begin{align}\label{eq:sigs2}
\blambda^{(i,\ell,u)} &:= \blambda^{(i,\ell)}(\bomega)     \qquad \text{for all  } \bomega \in \Omega_u \nonumber \\
\bepsilon^{(i,u)} &:= \bepsilon^{(i)}(\bomega)  \qquad \;\;\,\text{for all  } \bomega \in \Omega_u
\end{align}
when referring to these constants.

For simplicity of the exposition
let us temporarily assume that that the network has no bias, and let us ignore the vector $\ones$ appearing in equation (5).  Also let us define  the matrix   $T^{(i)}=\I - \ones  \otimes \yy^{(i)}$. 
The loss (5) then becomes:
$$
\Loss(W^{(1)}, \ldots, W^{(L)}, V)= \sum_i
\ones^T\sigma ( \; T^{(i)} V \sigma_\alpha(W^{(L)}   \ldots  \sigma_\alpha(W^{(2)} \sigma_\alpha( W^{(1)} \xx^{(i)} ) ))) 
$$
Since inside a cell $\Omega_u$ the activation pattern of the ReLU's or leaky ReLUs does not change, each $\sigma$ and $\sigma_\alpha$ in the above equation can be replaced by a diagonal matrix with $0$, $\alpha$ or ones in its diagonal. To be more precise, restricted to the cell $\Omega_u$, the loss can be written:
\begin{equation*}
\Loss |_{\Omega_u}(W^{(1)}, \ldots, W^{(L)}, V) =
 \sum_i
\ones^T   \mathcal{E}^{(i,u)}  T^{(i)} V  \Lambda^{(i,L,u)}W^{(L)}  \;  \ldots  \;  \Lambda^{(i,2,u)}W^{(2)}  \Lambda^{(i,1,u)} W^{(1)} \xx^{(i)}  
 \end{equation*}
 where $\mathcal{E}^{(i,u)}=\text{diag}(\bepsilon^{(i,u)}) $  and 
  $\Lambda^{(i,\ell,u)}=\text{diag}( \blambda^{(i,\ell,u)})$.
 From the above equation it is clear that  $\Loss |_{\Omega_u}$ is a multilinear form of its arguments.
%
 
 Going back to our case of interest, where we do have biases and where we do not ignore  the vector $\ones$, the picture becomes  slightly more complex: the loss restricted to a cell is now a sum of multilinear form rather than a single multilinear form.
The exact formula follows by carefully expanding 
\begin{align*}
& \Loss|_{\Omega_u}=-1 + \frac{1}{N}  \sum_i \ones^T\sigma \left( \;  T^{(i)} (  V \sigma_\alpha(W^{(L)} \sigma_\alpha(  \ldots W^{(2)} \sigma_\alpha( W^{(1)} \xx^{(i)} + \bb_1 )+ \bb_2 \ldots  )+\bb_L    )  + \cc) + \ones \right) \\
 &= -1 + \frac{1}{N}  \sum_i \ones^T  \mathcal{E}^{(i,u)}  \left( \;  T^{(i)} (  V \Lambda^{(i,L,u)}( W^{(L)}  \Lambda^{(i,L-1,u)}(  \ldots W^{(2)}\Lambda^{(i,1,u)}( W^{(1)} \xx^{(i)} + \bb_1 )+ \bb_2 \ldots  )+\bb_L    )  + \cc) + \ones \right)
\end{align*}
\end{proof}

\begin{corollary*}[Corollary 2 from the paper]  If $\bomega \in \Omega\setminus \mathcal{N}$ and the Hessian matrix $D^2 \Loss(\bomega)$ does not vanish, then it must have at least one strictly positive and one strictly negative eigenvalue.
\end{corollary*}

\begin{proof}
it suffices to note that a multilinear form
 $
 \phi: \real^{d_1} \times \ldots \times \real^{d_n} \to \real
 $ can always be written as
\begin{equation} \label{multi}
\phi(\vv_1,  \ldots, \vv_n) = \sum_{j_1=1}^{d_1} \ldots  \sum_{j_n=1}^{d_n} A_{j_1,\ldots,j_n}  v_{1,j_1} \ldots v_{n,j_n}
\end{equation}
for some tensor  $\{A_{j_1,\ldots,j_n}:  1\le j_k \le d_k \},$ with $v_{k,j}$ denoting the $j^{ {\rm th}}$ component of the vector $\vv_k$. From \eqref{multi} it is clear that  
$$\frac{\partial^2 \phi}{ \partial v_{k,j}^2}=0$$
and therefore the trace of the Hessian matrix of $\phi$ vanishes. Thus the (real) eigenvalues of the (real, orthogonally diagonalizable) Hessian sum to zero, and so if the Hessian is not the zero matrix then it has at least one strictly positive and one strictly negative eigenvalue.
\end{proof}

 \begin{corollary*}[Corollary 1 from the paper]
 Local minima and maxima of the loss  occur only on the boundary set $\mathcal{N}$ or on those cells $\Omega_u$ where the loss is constant. In the latter case, $\Loss |_{\Omega_u}(\bomega) = \phi_{L+2}^u$.
\end{corollary*}
\begin{proof}
The proof  relies on the so called {\it maximum principle} for harmonic functions.
We first recall that a  real valued function $f: \real^d \to \real$ is said to be harmonic at a point $x$ if it twice differentiable at $x$ and if its Laplacian vanishes at $x$, that is:
\begin{equation} \label{laplacian}
\textstyle \Delta f(x):=\sum_{i=1}^d \frac{\partial^2 f}{ \partial x_i^2}(x)=0
\end{equation}
Note that \eqref{laplacian} is equivalent to saying that the Hessian matrix $D^2f(x)$ has zero trace.
A function is said to be harmonic on an open set $\mathcal{O}$ if it harmonic at every points $x \in \mathcal O$. 
In the proof of the previous corollary we have shown that the trace of $D^2 \Loss(\bomega)$ is equal to zero if $\bomega \notin \N$, which exactly means that the loss is harmonic on the open set $\Omega\backslash \N$.

The strong maximum principle states that a non-constant harmonic function cannot
attain a  local minimum or  maximum at an interior point of an open connected set.  Therefore  if $\hat \bomega $ is a local minimum of $\Loss$ that occurs in a cell $\Omega_u$, then there exists a small neighborhood
$$
B_{\veps}( \hat \bomega) := \{ \bomega \in \Omega : \| \bomega - \hat \bomega\| < \veps \} \subset \Omega_u
$$
on which $\Loss|_{\Omega_u}(\bomega)$ is constant. Thus
\begin{equation} \label{loulou}
\Loss|_{\Omega_u}(\hat \bomega + \delta \bomega ) = \Loss|_{\Omega_u}(\hat \bomega)
\end{equation}
must hold for all $\bomega$ and all $\delta$ small enough. Now use the  multilinearity of $\Loss|_{\Omega_u}$ to expand the left-hand-side into powers of $\delta$:
\begin{equation}\label{eq:thisone}
\Loss|_{\Omega_u}(\hat \bomega + \delta \bomega ) = \Loss|_{\Omega_u}(\hat \bomega ) + \sum^{L}_{k=1} \delta^{k} f_{k}(\hat\bomega ; \bomega) + \delta^{L+1}\big( \phi^{u}_{0} + \phi^{u}_{1} \big)(\bomega).
\end{equation}
That $\big( \phi^{u}_{0} + \phi^{u}_{1} \big)(\bomega)$ is, in fact, the highest-order term come
from the fact that, among all the multilinear forms appearing in the statement of theorem 1, 
$\phi^{u}_{0}$ and $\phi^{u}_{1}$ are the only one having $L+1$ inputs.
The terms of order $k \le L$ appearing in the expansion \eqref{eq:thisone}  depends both  on the minimizer $\hat\bomega$ and the perturbation $\bomega$, and we denote them by   $f_{k}(\hat\bomega ; \bomega)$. 

Combining \eqref{loulou} and \eqref{eq:thisone} we have that
\begin{equation}\label{eq:thisone2}
 \sum^{L}_{k=1} \delta^{k} f_{k}(\hat \bomega;\bomega) + \delta^{L+1}\big( \phi^{u}_{0} + \phi^{u}_{1} \big)(\bomega)=0.
\end{equation}
 Since \eqref{eq:thisone2} must hold for all $\delta$ small enough, all like powers must vanish
$$
f_k(\hat \bomega;\bomega) = 0 \qquad \text{and} \qquad \phi^{u}_{0}(\bomega) + \phi^{u}_{1}(\bomega)=0.
$$
The second equation can be written as
$$
\phi_0^u(\omega^{(1)}, \omega^{(2)}, \omega^{(3)}, \omega^{(4)} \ldots, \omega^{(L)},V) = 
-\phi_1^u(\bb^{(1)}, \omega^{(2)}, \omega^{(3)}, \omega^{(4)} \ldots, \omega^{(L)},V) 
$$
Since  $\phi^{u}_{0}$ depends linearly on  $\omega^{(1)}$ whereas $\phi^{u}_{0}$ does not depend on  $\omega^{(1)}$, and
since  $\phi^{u}_{1}$ depends linearly on  $\bb^{(1)}$ whereas $\phi^{u}_{0}$ does not depend on  $\bb^{(1)}$, the only way for the above equality to hold for all perturbation $\bomega$ is that both functions are the zero function.
%
 Thus $\phi^u_0 + \phi^u_1$ is the zero function, and so $\phi^{u}_2$ is the highest-order multilinear form in the decomposition from theorem 1. This implies that
$$
f_{L}(\hat \bomega;\bomega) = \phi^u_2(\bb^{2},\omega^{(3)},\ldots,\omega^{(L)},V),
$$
actually just depends on $\bomega$, and that
 $f_{L}$ must vanish by \eqref{eq:thisone2}. Thus $\phi^u_2$ is the zero function as well. Continuing in this way shows that each $\phi^{u}_{\ell}$ is the zero function for $0 \leq \ell \leq L+1,$ and so in fact
$$
\Loss|_{\Omega_u}(\bomega) = \phi^{u}_{L+2}
$$
as claimed.
\end{proof}

\begin{theorem*}[Theorem 2  from the paper]
 If $\bomega$ is a type II local minimum then $\Loss(\bomega)>0$.
\end{theorem*}
\begin{proof}
We will show the contrapositive:  if $\Loss(\bomega)=0$, then $\bomega$ is a type I local minimum.
Given some $\lambda > 1$ and $ \veps^{(\ell)} > 0$ put
$$
\tilde \bomega := (\bomega^{1},\ldots,\bomega^{(L)},\lambda V,\bb^{(1)}+\veps^{(1)},\ldots,\bb^{(L)}+\veps^{(L)},\lambda \cc)
$$
and let $\tilde \xx^{(i,\ell)}$ denote the corresponding features at layer $\ell$ using these parameters. Then there exists some constant $C \geq 1$ so that
$$
\big|\Delta \xx^{(i,\ell)} \big| := \big|\tilde \xx^{(i,\ell)} - \xx^{(i,\ell)}  \big| \leq C\max\left\{ \veps^{(1)},\ldots,\veps^{(\ell)} \right\}.
$$
For each hidden layer $\ell \in [L]$ and each corresponding feature $k \in [d_{\ell}]$ define the activation levels
\begin{align*}
\eta^{(i,\ell)}_{k} &:= \langle W^{(\ell)}_k(\omega^{(\ell)}) , \xx^{(i,\ell-1)} \rangle + \bb^{(\ell)}_{k} \\
\tilde \eta^{(i,\ell)}_{k} &:= \langle W^{(\ell)}_k(\omega^{(\ell)}) , \tilde \xx^{(i,\ell-1)} \rangle + \bb^{(\ell)}_{k} + \veps^{(\ell)} = \eta^{(i,\ell)}_{k} + \langle W_k(\omega^{(\ell)}) , \Delta \xx^{(i,\ell-1)} \rangle + \veps^{(\ell)}
\end{align*}
for each set of parameters. Define
$$
\eta^{(-,\ell)} := \max_{ i \in [N], k \in [d_{\ell}] } \; \{ \eta^{(i,\ell)}_{k} : \eta^{(i,\ell)}_{k} < 0 \}
$$
with the convention that $\eta^{(-,\ell)} = -\infty$ if $ \eta^{(i,\ell)}_{k}  \geq 0$ for all $k\in [d_{\ell}],i \in [N]$. Set
$$
\eta^{(-)} := \max_{ \ell \in [L] } \; \eta^{(-,\ell)}
$$
as the largest negative activation level of the network. Put
$$
W_{*} := \max_{ \ell \in [L} \} \; \|W^{(\ell)}\|_{2}
$$
as the largest norm of the matrices $W^{(\ell)}$ across the network. Take any $0 < \veps^{(L)}$ so that
$$
\veps^{(L)} < \frac{|\eta^{(-)}|}{ 2C\max\{1,W_* \} },
$$
then for $\ell = L-1,\ldots,1$ inductively choose
$$
0 < \veps^{(\ell)} < \min\left\{ \frac{|\eta^{(-)}|}{ 2C\max\{1, W_* \} } , \frac{ \veps^{(\ell+1)} }{C \max\{ 1 , W_*\} } \right\}.
$$
In particular, $\veps^{(L)} > \veps^{(L-1)} > \cdots > \veps^{(1)} > 0$. For any such choice
\begin{align*}
\tilde \eta^{(i,\ell)}_{k} &\geq \veps^{(\ell)} - |\langle W^{(\ell)}_{k} , \Delta \xx^{(i,\ell-1)} \rangle | \geq \veps^{(\ell)} - C \max\{1,W_*\} \veps^{(\ell-1)} > 0 \;\;\,\quad \text{if} \quad \eta^{(i,\ell)}_{k} \geq 0\\
\tilde \eta^{(i,\ell)}_{k} &\leq -\big|\eta^{(-)}\big| + C \max\{1,W_*\} \veps^{(\ell-1)} + \veps^{(\ell)} \leq -\frac1{2}\big|\eta^{(-)}\big| + \veps^{(\ell)} < 0 \quad \text{if} \quad \eta^{(i,\ell)}_{k} < 0.
\end{align*}
Now let $\cl(\x^{(i)}) \in \{1,\ldots,R\}$ denote the class of a data point. Put
\begin{align*}
\zz^{(i,r)}(\bomega) &:= \langle \vv_r - \vv_{\cl(\x^{(i)})} , \xx^{(i,L)}(\bomega) \rangle + (c_{r} - c_{\cl(\x^{(i)})}) \\
\tilde \zz^{(i,r)}(\tilde \bomega) &:= \lambda \langle \vv_r - \vv_{\cl(\x^{(i)})} , \xx^{(i,L)}(\tilde \bomega) \rangle + \lambda(c_{r} - c_{\cl(\x^{(i)})}) = \lambda \zz^{(i,r)} + \lambda \langle \vv_r - \vv_{\cl(\x^{(i)})} , \Delta \xx^{(i,L)} \rangle
\end{align*}
as the outputs of the network. That $\Loss(\bomega) = 0$ implies
$$
\sigma\big( 1 + \zz^{(i,r)} \big) = 0 \quad \text{for all} \quad r \neq \cl(\x_i)
$$
and so $\zz^{(i,r)} \leq -1$ for all $i,r\neq \cl(\x_i)$. Thus
$$
\tilde \zz^{(i,r)} \leq - \lambda +  2C\|V\|_{2}\lambda \veps^{(L)},
$$
and so if
$$
0 < \veps^{(L)} < \frac{ \lambda - 1 }{ 2C\|V\|_{2}\lambda }
$$
then $\tilde \zz^{(i,r)} < -1$ for all $i, r \neq \cl(\x_i)$. For any $\lambda >1$ and corresponding choices of $\veps^{(\ell)}$ it follows that $\Loss(\tilde \bomega) = 0$. Moreover, by construction the signature function $\mathcal{S}(\tilde \bomega)$ is constant as $\lambda,\veps^{(\ell)}$ vary and so $\tilde \bomega$ lies in some fixed cell $\Omega_u$ for all $\lambda > 1$. But $\tilde \bomega \to \bomega$ as $\lambda \to 1$, so $\bomega \in \overline \Omega_u$. Finally, for this cell $\Omega_u$ it follows that $\Loss|_{\Omega_u} = 0$ by definition of the signature function.
\end{proof}

\begin{theorem*}[Theorem 3  from the paper]
Consider the loss  (5) for a fully connected network. Assume that $\alpha>0$ and that the data points $(\xx^{(i)},\yy^{(i)})$ are generic. Then $\Loss(\bomega)=0$ at any type I local minimum.
\end{theorem*}

In order to prove this theorem we will need  two lemmas.

\begin{lemma}\label{lem:dumdumlemma}
Let $\x_i,i\in[N]$ denote arbitrary points in $\R^{d}$ and $D_{i},i \in [N]$ an arbitrary family of $m \times m$ diagonal matrices. Then
\begin{align*}
&\sum^{N}_{i=1} D_{i}Ax_{i} = 0 \quad \text{for all} \quad A \in \M_{m\times d} \quad \text{if and only if}\\
&\sum^{N}_{i=1} D_{i}(j,j)x_i = 0 \quad \text{for all} \quad j \in [m]
\end{align*}
\end{lemma}
\begin{proof}
Letting $A = [\aaa_1 , \ldots , \aaa_d]$ and writing out the expression
$$
\sum^{N}_{i=1} D_{i}Ax_{i} = 0
$$
column-wise shows that the first statement is equivalent to
$$
\sum^{d}_{k=1} \left( \sum^{N}_{i=1} D_{i}  x_{i}(k) \right) \aaa_{k}= 0 \quad \text{for all} \quad (\aaa_1,\ldots,\aaa_{d}) \in \R^{d},
$$
which in turn is equivalent to
\begin{equation}\label{eq:stupiddiag}
\left( \sum^{N}_{i=1} D_{i}  x_{i}(k) \right) = 0 \quad \text{for all} \quad k \in [d]
\end{equation}
For each $k$ the left hand side determines a diagonal matrix $E_k$, which vanishes if and only if all of its diagonal entries vanish. Thus \eqref{eq:stupiddiag} is equivalent to
$$
\sum^{N}_{i=1} D_{i}(j,j) x_{i}(k) = 0 \quad \text{for all} \quad k \in [d], j \in [m],
$$
which is exactly the conclusion of the lemma written component-by-component.
\end{proof}

\begin{lemma}
Let $\Omega_u$ denote a cell on which the loss $\Loss(\bomega)$ of a fully-connected network is constant. Let $\veps^{(i,u)}_{r} \in \{0,1\}$ denote the error indicator for $\xx^{(i)}$ and class $r \in [R]$ on the cell. Define
$$
\veps^{(i,u)} := \sum_{ r : \yy^{(i)}_r = 0} \veps^{(i,u)}_r
$$
as the total number of errors for $\xx^{(i)}$ on the cell. Then there exist scalars $\lambda^{(i)} \in \{1,\alpha,\ldots,\alpha^{L} \}$ so that the equalities
\begin{align}\label{eq:weirdcond}
\sum_{ \yy^{(i)}_{r} = 1 } \, \veps^{(i,u)} \lambda^{(i)}\mu^{(i)}\xx^{(i)} &= \sum_{ \yy^{(i)}_{r} = 0 } \, \veps^{(i,u)}_r \lambda^{(i)} \mu^{(i)}\xx^{(i)} \nonumber \\
\sum_{ \yy^{(i)}_{r} = 1 } \, \veps^{(i,u)} \lambda^{(i)}\mu^{(i)} &= \sum_{ \yy^{(i)}_{r} = 0 } \, \veps^{(i,u)}_r \lambda^{(i)}\mu^{(i)} \nonumber \\
\sum_{ \yy^{(i)}_{r} = 1 } \, \veps^{(i,u)} \mu^{(i)} &= \sum_{ \yy^{(i)}_{r} = 0 } \, \veps^{(i,u)}_r\mu^{(i)}
\end{align}
hold for all $r \in [R]$.
\end{lemma}
\begin{proof}
The $\Loss|_{\Omega_u}$ is constant if and only if all of the multilinear forms in the decomposition of $\Loss|_{\Omega_u}$ vanish, that is, each of the multilinear forms is the zero function (c.f. the proof of theorem 1). Let $\cl(\xx^{(i)})$ denote the class of a data point. Put
$$
\boldsymbol{\mu}^{(i)} := ( \mu^{(i)},\ldots,\mu^{(i)}) \in \R^{R} \quad  \qquad T^{(i)} := \mathrm{Id} - \ones \otimes \yy^{(i)}
$$
and for such a cell define the diagonal matrices
\begin{align*}
\veps^{(i,u)}_{r} &:= \sigma^{\prime}\left( 1 + \big( \langle \vv_r - \vv_{\cl(\xx^{(i)})}  , \xx^{(i,L)} \rangle + c_r - c_{\cl(\xx^{(i)})} \big) \right)\quad \lambda^{(i,\ell,u)}_{k} := \sigma^{\prime}_{\alpha}\left( \langle \aaa^{(\ell)}_{k} ,\xx^{(i,\ell-1)} \rangle + \bb^{(\ell)}_{k} \right)\\
\mathcal{E}^{(i,u)} &:= \mathrm{diag}\big( \veps^{(i,u)}_{1},\ldots,\veps^{(i,u)}_{R} \big) \qquad \qquad \qquad \qquad \quad \qquad \quad
\Lambda^{(i,\ell,u)} :=  \mathrm{diag}\big(\lambda^{(i,\ell,u)}_{1},\ldots,\lambda^{(i,\ell,u)}_{d_\ell}\big).
\end{align*}
These diagonal matrices remain constant for all $\bomega \in \Omega_u$ by the definition of a cell. Set
$$
\boldsymbol{\nu}^{(i,u)} := (T^{(i)})^{T}\mathcal{E}^{(i,u)}\boldsymbol{\mu}^{(i)}.
$$
Then the multilinear forms in theorem 1 have the following form
\begin{align*}
\phi^{u}_{0}\big( A^{(1)},\ldots,A^{(L)},V\big) &= \sum^{N}_{i=1} \langle V\Lambda^{(i,L,u)}A^{(L)}\cdots A^{(2)}\Lambda^{(i,1,u)}A^{(1)}\xx^{(i)} ,  \boldsymbol{\nu}^{(i,u)}\rangle \\
\phi^{u}_{\ell}\big( \bb^{(\ell)},A^{(\ell+1)},\ldots,A^{(L)},V) &= \sum^{N}_{i=1} \langle V\Lambda^{(i,L,u)}A^{(L)}\cdots A^{(\ell+1)}\Lambda^{(i,\ell,u)}\bb^{(\ell)} ,  \boldsymbol{\nu}^{(i,u)}\rangle \\
\phi^{u}_{L}(\bb^{(L)},V) &= \sum^{N}_{i=1} \langle V \Lambda^{(i,L,u)}\bb^{(L)},\boldsymbol \nu^{(i,u)} \rangle \\
\phi^{u}_{L+1}(\cc) &= \sum^{N}_{i=1} \langle \cc , \boldsymbol \nu^{(i,u)} \rangle.
\end{align*}
Now $\phi^{u}_{0}$ vanishes for all $V,A^{(\ell)}$ if and only if
$$
\sum^{N}_{i=1} \big(\Lambda^{(i,L,u)}A^{(L)}\cdots A^{(2)}\Lambda^{(i,1,u)}A^{(1)}\xx^{(i)} \big) \otimes \boldsymbol{\nu}^{(i,u)} = 0
$$
for all $A^{(\ell)}$. In other words, each of the $R$ columns
$$
\sum^{N}_{i=1} \big(\Lambda^{(i,L,u)}A^{(L)}\cdots A^{(2)}\Lambda^{(i,1,u)}A^{(1)}\xx^{(i)} \big)\boldsymbol{\nu}^{(i,u)}_{r} = 0 \quad \text{for all} \quad r \in [R]
$$
must vanish. Now $\Lambda^{(i,L,u)}$ is diagonal, and so lemma \ref{lem:dumdumlemma} shows this can happen if and only if
$$
\sum^{N}_{i=1}\lambda^{(i,L,u)}_{k_L} ( \Lambda^{(i,L-1,u)}A^{(L-1)}\cdots A^{(2)}\Lambda^{(i,1,u)}A^{(1)}\xx^{(i)} \big)\boldsymbol{\nu}^{(i,u)}_{r} = 0 \quad \text{for all} \quad (r,k_L) \in [R] \times [d_L].
$$
Now $\lambda^{(i,L,u)}_{k_L}\Lambda^{(i,L-1,u)}$ is once again diagonal, and so this can happen if and only if
$$
\sum^{N}_{i=1}\lambda^{(i,L,u)}_{k_L}\lambda^{(i,L-1,u)}_{k_{L-1}} ( \Lambda^{(i,L-2,u)}A^{(L-2)}\cdots A^{(2)}\Lambda^{(i,1,u)}A^{(1)}\xx^{(i)} \big)\boldsymbol{\nu}^{(i,u)}_{r} = 0 \quad \text{for all} \quad (r,k_L,k_{L-1}) \in [R] \times [d_L] \times [d_{L-1}].
$$
Continuing in this way shows $\phi^{u}_{0}$ vanishes if and only if
$$
\sum^{N}_{i=1} \boldsymbol{\nu}^{(i,u)}_{r}\big( \lambda^{(i,L,u)}_{k_L}\lambda^{(i,L-1,u)}_{k_{L-1}}\cdots\lambda^{(i,1,u)}_{k_1}\big) \xx^{(i)} = 0 \quad \text{for all} \quad (r,k_L,\ldots,k_1) \in [R]\times[d_L]\times\cdots\times[d_1].
$$
A similar argument shows $\phi^{u}_{\ell}, \ell \in [L]$ vanish if and only if
$$
\sum^{N}_{i=1} \boldsymbol{\nu}^{(i,u)}_{r}\lambda^{(i,L,u)}_{k_L}\cdots \lambda^{(i,\ell,u)}_{k_{\ell}} = 0 \quad \text{for all} \quad (r,k_L,\ldots,k_\ell) \in [R]\times[d_L]\times\cdots\times[d_\ell],
$$
while $\phi^u_{L+1}$ vanishes if and only if
$$
\sum^{N}_{i=1} \boldsymbol{\nu}^{(i,u)}_{r} = 0 \quad \text{for all} \quad r \in [R].
$$
Pick some $(k_L,\ldots,k_1) \in [d_{L}]\times \cdots \times [d_{1}]$ arbitrary and set
$$
\lambda^{(i)} := \lambda^{(i,L,u)}_{k_L} \cdots \lambda^{(i,1,u)}_{k_1} \in \{1,\alpha,\ldots,\alpha^{L}\}.
$$
For any such a choice the fact that $\phi^{u}_{0}$ vanishes implies the first equation in \eqref{eq:weirdcond} due to the definition of $\boldsymbol{\nu}^{(i,u)}_{r}$, while the fact that $\phi^{u}_{1}$ vanishes implies the second equation. Finally, that the third equation in \eqref{eq:weirdcond} holds comes from the fact that $\phi^{u}_{L+1}$ vanishes.
\end{proof}

\begin{proof}[Proof of the theorem]
The claims are immediate from the preceeding two lemmas. For generic data
\begin{align*}
\sum_{ \yy^{(i)}_{r} = 1 } \, \veps^{(i,u)} \lambda^{(i)}\mu^{(i)}\xx^{(i)} &= \sum_{ \yy^{(i)}_{r} = 0 } \, \veps^{(i,u)}_r \lambda^{(i)} \mu^{(i)}\xx^{(i)} \nonumber \\
\sum_{ \yy^{(i)}_{r} = 1 } \, \veps^{(i,u)} \lambda^{(i)}\mu^{(i)} &= \sum_{ \yy^{(i)}_{r} = 0 } \, \veps^{(i,u)}_r \lambda^{(i)}\mu^{(i)} \nonumber \\
\sum_{ \yy^{(i)}_{r} = 1 } \, \veps^{(i,u)} \mu^{(i)} &= \sum_{ \yy^{(i)}_{r} = 0 } \, \veps^{(i,u)}_r\mu^{(i)}
\end{align*}
can hold only if all of the coefficients vanish, i.e. $\veps^{(i,u)}_r \lambda ^{(i)} = 0$ for all $i \in [N],r \in [R]$. But $\lambda^{(i)} > 0$ for all $i \in [N]$ since $\alpha  >0$, and so $\veps^{(i,u)}_r = 0$ for all $i \in [N] , r \in [R]$. That is, $\Loss(\bomega) = 0$ on any flat cell $\Omega_u$.
\end{proof}

\section*{Proofs of  theorems from  section 3: Critical Point Analysis}
We begin by showing that no sub-optimal minimizers exist in the simplest case $\alpha=1$, i.e. for deep linear networks with binary hinge loss. The loss here is
\begin{equation} \label{vvvv}
\sum\mu^{(i)} \sigma \left( 1 - y^{(i)}\big(\langle \vv ,  W^{(L)} \cdots W^{(1)} \xx^{(i)} \rangle +c \big)    \right).
\end{equation}
Note that if we define $\bar \vv := (W^{(L)}\cdots W^{(1)})^{T}\vv$ then \eqref{vvvv} corresponds to a convex loss $E(\bar \vv , c )$ whose first argument is parametrized as a multilinear product.
\begin{theorem}[Deep Linear Networks]
Consider the loss \eqref{vvvv} with arbitrary data and assume that $\bomega$ is any critical point in the Clarke sense. Then the following hold ---
\begin{enumerate}[label=(\roman*)]
\item If $\bar \vv \neq \0$ then $\bomega$ is a global minimum. 
\item If the $\mu^{(i)}$ weight both classes equally weighted and $\bomega$ is a local minimum with $\bar \vv = \0$ then it is a global minimum.
\end{enumerate}
\end{theorem}
\begin{proof}
Recall that the loss takes the form
$$
\Loss(\bomega) = \sum^{N}_{i=1} \mu^{(i)}\sigma\big( 1 -  y^{(i)}( \langle \vv , \xx^{(i,L)} \rangle + c ) \big) \qquad y^{(i)} \in \{-1,1\}
$$
for a deep linear network, where
$$
\xx^{(i,L)} = W^{(L)} \cdots W^{(1)} \xx^{(i)}
$$
denote the features from the linear network at the $L^{ {\rm th}}$ hidden layer. Let $W := W^{(L)} \cdots W^{(1)}$ and $\bar \vv := W^{T} \vv,$ and recall that on a cell $\Omega_u$ the expression
$$
\Loss|_{ \overline \Omega_u}(\bomega) = \sum^{N}_{i=1} \mu^{(i)}\veps^{(i,u)} - \sum^{N}_{i=1} \mu^{(i)} \veps^{(i,u)} y^{(i)}\big( \langle \vv , \xx^{(i,L)} \rangle + c \big)
$$
defines the loss. The expressions
\begin{align*}
-\nabla_{ W^{(1)} } \Loss|_{ \overline \Omega_u} &= \sum^{N}_{i=1} (W^{(L)}\cdots W^{(2)} )^{T}\vv \otimes \xx^{(i)} \mu^{(i)} \veps^{(i,u)}y^{(i)}  \\
-\nabla_{ c } \Loss|_{ \overline \Omega_u} &= \sum^{N}_{i=1} \mu^{(i)} \veps^{(i,u)}y^{(i)} 
\end{align*}
then furnish the gradient of $\Loss$ with respect to the parameters $(W^{(1)},c)$ on the cell. By definition, it therefore follows that
\begin{align*}
\0 &= (W^{(L)}\cdots W^{(2)})^{T} \vv \otimes \zz \qquad \zz := \sum^{N}_{i=1} \mu^{(i)}\lambda^{(i)} y^{(i)} \xx^{(i)} , \quad \lambda^{(i)} := \sum_{u \in \mathcal{I}(\bomega)} \theta^{(u)} \veps^{(i,u)} \\
0 &= \sum^{N}_{i=1} \mu^{(i)}\lambda^{(i)}y^{(i)}
\end{align*}
a critical point, where the non-negative constants $\theta^{(u)}\geq 0$ sum to one. In particular, the coefficients $\lambda^{(i)} \in [0,1]$ lie in the unit interval. Multiplying by transpose of $W^{(1)}$ shows
$$
\0 = \bar \vv \otimes \zz
$$
at a critical point. This gives a dichotomy --- either $\zz$ vanishes or $\bar \vv$ vanishes.

Consider first the case where $\zz$ vanishes. Then
\begin{equation}\label{eq:subdifferyay}
\0 = \sum^{N}_{i=1} \mu^{(i)} \lambda^{(i)}y^{(i)} \xx^{(i)}  \qquad \text{and} \qquad 0 = \sum^{N}_{i=1} \mu^{(i)} \lambda^{(i)}y^{(i)}
\end{equation}
both hold. Moreover, for each $i \in [N]$ the coefficients $\lambda^{(i)}$ obey
\begin{equation}\label{eq:lamvalinter}
\lambda^{(i)} \in \begin{cases}
\{1\} & \text{if} \quad 1 - \big( \langle \bar \vv , \xx^{(i)} \rangle  + c \big) y^{(i)} > 0\\
[0,1] & \text{if} \quad 1 - \big( \langle \bar \vv , \xx^{(i)} \rangle  + c \big) y^{(i)} = 0\\
\{0\} & \text{if} \quad 1 - \big( \langle \bar \vv , \xx^{(i)} \rangle  + c \big) y^{(i)} < 0
\end{cases}
\end{equation}
as well. This follows from the observation that $\veps^{(i,u)} = 1$ for all cells $u \in \mathcal{I}(\bomega)$ in the first case, while $\veps^{(i,u)} = 0$ for all cells $u \in \mathcal{I}(\bomega)$ in the last case. Now for each $i \in [N]$ define the functions
$$
f^{(i)}( \ww , d ) := \mu^{(i)}\sigma\big( 1 - \big( \langle \ww , \xx^{(i)} \rangle  + d \big) y^{(i)} \big),
$$
which are clearly convex. The subdifferential $\partial f^{(i)}(\ww,d)$ of $f^{(i)}( \ww , d )$ at a point $(\ww,d)$ is easily computed as
$$
\partial f^{(i)}(\ww,d) = \begin{cases}
\{ -\mu^{(i)}y^{(i)}( \xx^{(i)} , 1 )^{T} \} & \text{if} \quad 1 - \big( \langle \ww , \xx^{(i)} \rangle  + d \big) y^{(i)} > 0\\
-[0,1]\mu^{(i)}y^{(i)}( \xx^{(i)} , 1 )^{T}& \text{if} \quad 1 - \big( \langle \ww , \xx^{(i)} \rangle  + d \big) y^{(i)} = 0\\
\{\0\} & \text{if} \quad 1 - \big( \langle \ww , \xx^{(i)} \rangle  + d \big) y^{(i)} < 0,
\end{cases}
$$
and so (\ref{eq:subdifferyay},\ref{eq:lamvalinter}) imply that the inclusion
$$
\0 \in \sum^{N}_{i=1} \partial f^{(i)}(\bar \vv , c)
$$
holds. As each $f^{(i)}(\ww,d)$ is Lipschitz, the composite function
$$
E(\ww,d) := \sum^{N}_{i=1} f^{(i)}(\ww,d)
$$
obeys the calculus rule
$$
\partial E(\ww,d) =  \sum^{N}_{i=1} \partial f^{(i)}(\ww,d)
$$
(c.f. \cite{borwein2010convex} theorem 3.3.5). Thus $\0 \in \partial  E(\bar \vv , c)$ and so $(\bar \vv,c)$ is a global minimizer.

It remains to address the case where $\bar \vv$ vanishes. Then as a function of $c$ the loss remains constant for all $c$ in the unit interval,
$$
E( \0 , c ) = \sum^{N}_{i=1} \mu^{(i)} \sigma\big( 1  - y^{(i)}c \big) = \frac{\sigma(1-c) + \sigma(1+c)}{2} = 1,
$$
and moreover if $c \notin [-1,1]$ then the convex function $E( \0 , c )$ attains its minimum in the unit interval. This follows from the equal mass hypothesis
$$
\sum_{ y^{(i)} = 1 } \mu^{(i)} = \sum_{ y^{(i)} = -1 } \mu^{(i)} = \frac1{2}
$$
on the weights. At a local minimum $\bomega$ where $\bar \vv$ vanishes the parameter $c$ must therefore lie in the unit interval. It therefore suffices to assume that $c \in (-1,1)$ without loss of generality. But then the loss $\Loss$ is differentiable (in fact smooth) near $\bomega,$ and so theorem 1 from \cite{laurent2017deep} yields the result in this case.
\end{proof}

 \begin{theorem*}[Theorem 4  from the paper] 
Consider the loss (9) with  $\alpha>0$ and data $\xx^{(i)}, i \in [N] $ that are linearly separable. Assume that $\bomega=(W,\vv,\bb,c)$ is any critical point of the loss in the Clarke sense. Then either $\vv = \0$ or $\bomega$ is a global minimum.
\end{theorem*}

The proof of theorem 4 relies on the following pair of auxiliary lemmas. The former gives an explicit description of the multilinear decomposition for the loss in a network with one hidden layer; the latter computes the Clarke subdifferential in terms of the decomposition of parameter space into cells.

\begin{lemma}[Decomposition with $L=1$]  \label{one_hidden}
Let
\begin{align*}
\Loss |_{\Omega_u}(W,\vv,\bb,c)=
    \phi_0^u(W, \vv)
+ \phi_1^{u}(\bb,\vv)
+ \phi_2^{u}(c)
+ \phi_3^u 
\end{align*}
denote the loss on a cell $\Omega_u$. For $k \in [K]$ define
\begin{align*}
&\aaa_k^{(u)} := \sum_{i} \mu^{(i)}  y^{(i)} \varepsilon^{(i,u)}  \lambda_k^{(i,u)} \xx^{(i)}    \qquad \alpha_k^{(u)} :=  \sum_{i} \mu^{(i)}  y^{(i)} \varepsilon^{(i,u)}  \lambda_k^{(i,u)} \\
&\gamma^{(u)} :=\sum_{i} \mu^{(i)}  y^{(i)} \varepsilon^{(i,u)}    \qquad \qquad \quad \;\;\,\, \delta^{(u)} :=  \sum_{i} \mu^{(i)} \varepsilon^{(i,u)}. 
\end{align*}
Then $\phi^{u}_{3} = \delta^{(u)}$ and $\phi_2^u(c)=      - \gamma^{(u)}   c,$ while the relations
\begin{equation*}
 \phi_0^u( W, \vv)= -\sum_k     v_k  \langle  \aaa^{(u)}_k ,  \ww_k  \rangle,\quad \text{and} \quad
\phi_1^u( \bb,\vv)= -\sum_k     v_k     \alpha^{(u)}_k     b_k  \quad 
\end{equation*}
furnish the multilinear forms defining the loss on $\Omega_u$.
\end{lemma}
\begin{proof}
Restricted to a cell $\Omega_u$ the loss can be written as
\begin{align*}
\Loss|_{\Omega_u}(W,\vv,\bb,c)&=\sum_i  \mu^{(i)} \sigma \left[ \;  -y^{(i)} \left\{ \vv^T  \sigma_{\alpha} ( W  \xx^{(i)}+\bb)  + c     \right\}  + 1  \right]\\
&= \sum_i \mu^{(i)}  \varepsilon^{(i,u)} \left[ \;  -y^{(i)} \left\{ \vv^T  \Lambda^{(i,u)} ( W  \xx^{(i)}+\bb)  + c     \right\}  + 1  \right].
\end{align*}
Expanding this expression leads to
\begin{align*}
\Loss|_{\Omega_u}&
= \sum_i\mu^{(i)}  \varepsilon^{(i,u)} \left[ \;  -y^{(i)}\vv^T  \Lambda^{(i,u)} W  \xx^{(i)} -y^{(i)}  \vv^T  \Lambda^{(i,u)} \bb   -y^{(i)} c      + 1  \right] \\
&= \sum_i  -\mu^{(i)}\varepsilon^{(i,u)}  y^{(i)}\vv^T  \Lambda^{(i,u)} W  \xx^{(i)} -\mu^{(i)}\varepsilon^{(i,u)}y^{(i)}  \vv^T  \Lambda^{(i,u)} \bb   - \mu^{(i)}\varepsilon^{(i,u)} y^{(i)} c      + \mu^{(i)} \varepsilon^{(i,u)}  \\
 &=  \phi_0^{{(u)}}( W,\vv) + \phi_1^{(u)}(\bb,\vv) + \phi_2^{(u)}(c) + \phi_3^{(u)}
\end{align*}
Now let $\ww_k$ denote the $k^{{\rm th}}$ row of the matrix $W$ and note that $\vv^T  \Lambda^{(i,u)} W = \sum_{k} v_k \lambda_k^{(i,u)} \ww_k^T$
to find
\begin{align*}
\phi_0^{(u)}( W,\vv) &=
-\sum_i \mu^{(i)}  \varepsilon^{(i,u)}  y^{(i)}\left( \sum_{k} v_k \lambda_k^{(i,u)} \ww^T_k \right)  \xx^{(i)}\\
&=-\sum_i  \sum_{k} \mu^{(i)} \varepsilon^{(i,u)}  y^{(i)}  v_k \lambda_k^{(i,u)} \langle  \ww_k ,  \xx^{(i)} \rangle \\
&= -  \sum_{k} v_k   \left\langle \sum_i  \mu^{(i)}\varepsilon^{(i,u)} y^{(i)} \lambda_k^{(i,u)} \xx^{(i)},  \ww_k \right\rangle \\
&= -  \sum_{k} v_k  \langle \aaa_k^{(u)},  \ww_k \rangle.
\end{align*}
A similar argument reveals  
\begin{align*}
\phi_1^{(u)}( \bb,\vv) &=- \sum_i \mu^{(i)}
\varepsilon^{(i,u)}y^{(i)}  \vv^T  \Lambda^{(i,u)} \bb \\
&=- \sum_i \mu^{(i)}
\varepsilon^{(i,u)}y^{(i)}   \sum_k\left( v_k  \lambda^{(i,u)}_k b_k \right)\\
&=- \sum_k   v_k  \left(  \sum_i \mu^{(i)}
\varepsilon^{(i,u)}y^{(i)}    \lambda^{(i,u)}_k \right) b_k \\
&=-  \sum_k   v_k  \alpha^{(u)}_k b_k \\
\phi_2^{(u)}(c) &= - \gamma^{(u)} c
\qquad 
\gamma^{(u)} :=
-  \sum_i \mu^{(i)}\varepsilon^{(i,u)} y^{(i)}
\end{align*}
as claimed.
\end{proof}

\begin{lemma}[Subdifferential Calculation]\label{lem:subdiffcalc}
Fix a point $\bomega \in \Omega$ and let $\mathcal{I}(\bomega)$ denote its incidence set. Then the convex hull
\begin{equation}\label{eq:clarkfin}
\partial_0 \Loss(\bomega) = \left\{ \sum_{ u \in \mathcal{I}(\bomega) } \theta^{(u)} \nabla \Loss|_{\overline\Omega_u}(\bomega) : \theta^{(u)}\geq 0,\; \sum_{ u } \theta^{(u)} = 1 \right\}
\end{equation}
is the Clarke subdifferential of the loss $\Loss$ at $\bomega$. In particular, if $\mathcal{I}(\bomega) = \{u\}$ is a singleton then $\partial_0 \Loss(\bomega)$ is single-valued: $\partial_0 \Loss(\bomega) = \{\nabla \Loss|_{\Omega_u}(\bomega)\}$.
\end{lemma}
\begin{proof}
For a given point $\bomega \in \Omega$ recall that $\mathcal{I}(\bomega)$ denotes the set of indices of the cells that are adjacent to the point $\bomega$,
$$
\mathcal{I}(\bomega) := \left\{ u \in \{-1,1\}^{ND} : \bomega \in \overline\Omega_{u} \right\}
$$
where $\overline\Omega_u$ stands for the closure of the cell $\Omega_u$. Assume   $\bomega_k \to \bomega$ and $\bomega_k \notin \mathcal{N}$. As $\mathcal{I}(\bomega)$ is clearly finite it suffices to assume, by passing to a subsequence if necessary, that $\bomega_k \in \Omega_u$ for some $u \in \mathcal{I}(\bomega)$ and all $k$ sufficiently large. But then $\nabla \Loss(\bomega_k) = \nabla \Loss|_{\Omega_u}(\bomega_k)$,  and since
$\nabla \Loss|_{\Omega_u}$ is a continuous function (i.e. a sum of multilinear gradients) the limit $\nabla \Loss(\bomega_k)  \to \nabla \Loss|_{\overline \Omega_u}(\bomega)$ follows. As $\mathcal{N}$ has measure zero definition 2 reveals that
  the convex hull
\begin{equation*}
\partial_0 \Loss(\bomega) = \left\{ \sum_{ u \in \mathcal{I}(\bomega) } \theta^{(u)} \nabla \Loss|_{\overline\Omega_u}(\bomega) : \theta^{(u)}\geq 0,\; \sum_{ u } \theta^{(u)} = 1 \right\}
\end{equation*}
is the Clark subdifferential at  $\bomega$. In particular,  if $\bomega \in \Omega_u$ for some $u$ then $\mathcal{I}(\bomega) = \{u\}$  and  \eqref{eq:clarkfin} simply becomes $\partial_0 \Loss(\bomega)=\{
 \nabla \Loss|_{\Omega_u}(\bomega)\}$ as expected.
\end{proof}

\begin{proof}[Proof of Theorem 6]
By assumption $\0 \in \partial_{0}\Loss(\bomega)$ and so
\begin{equation} \label{pipi}
\0 = \sum_{ u \in \mathcal{I}(\bomega) }  \theta^{(u)} \nabla \Loss|_{\overline\Omega_u}(\bomega)
\end{equation}
for some collection of positive coefficients $\theta^{(u)}$ due to the characterization \eqref{eq:clarkfin} of the subdifferential. The explicit formulas from lemma \ref{one_hidden} show
$$
\frac{\partial\Loss|_{\Omega_u}}{\partial \ww_k}(\bomega)= - v_k \aaa_k^{(u)} \qquad \text{and} \qquad 
\frac{\partial\Loss|_{\Omega_u}}{\partial b_k}(\bomega)=-v_k \alpha_k^{(u)},
$$
which by \eqref{pipi} then obviously implies that both
$$
 \sum_{ u \in \mathcal{I}(\bomega) } \theta^{(u)}  v_k \aaa_k^{(u)}= \0  \qquad \text{and} \qquad 
  \sum_{ u \in \mathcal{I}(\bomega) }  \theta^{(u)}  v_k \alpha_k^{(u)}=0 
$$
must hold for all $k$. Substituting the expressions $\aaa_k^{(u)}$ and $b_k$ provided in lemma 2 then gives the equalities
\begin{align*}
\0 &= v_k\left( \sum_{u} \sum_{ i : y^{(i)} = 1 } \theta^{(u)} \mu^{(i)}\veps^{(i,u)} \lambda^{(i,u)}_k\xx^{(i)} - \sum_{u} \sum_{ i : y^{(i)} = -1 } \theta^{(u)} \mu^{(i)}\veps^{(i,u)} \lambda^{(i,u)}_k\xx^{(i)}\right)\\
0 &= v_k\left( \sum_{u} \sum_{ i : y^{(i)} = 1 } \theta^{(u)} \mu^{(i)} \veps^{(i,u)} \lambda^{(i,u)}_k - \sum_{u} \sum_{ i : y^{(i)} = -1 } \theta^{(u)} \mu^{(i)} \veps^{(i,u)} \lambda^{(i,u)}_k\right)
\end{align*}
 If there exists a $k$ for which $v_k \neq 0$ then an interchange of summations reveals
\begin{align} \label{s1}
\sum_{ i : y^{(i)} = 1 } \vrho^{(i)}_k \xx^{(i)} &= \sum_{ i : y^{(i)} = -1 } \vrho^{(i)}_k \xx^{(i)}\\
\sum_{ i : y^{(i)} = 1 } \vrho^{(i)}_k  &= \sum_{ i : y^{(i)} = -1 } \vrho^{(i)}_k \qquad \text{where} \qquad \vrho^{(i)}_{k} := \sum_{u} \theta^{(u)} \mu^{(i)}\veps^{(i,u)} \lambda^{(i,u)}_k  \label{s2}
\end{align}
The claim then follows since (\ref{s1},\ref{s2}) cannot hold unless all the $\vrho^{(i)}_k$ vanish. To see this, 
note that if the $\vrho^{(i)}_k$ do not vanish then 
$$ 0 < Q :=\sum_{ i : y^{(i)} = 1 } \vrho^{(i)}_k = \sum_{ i : y^{(i)} = -1 } \vrho^{(i)}_k,$$
and upon dividing by $Q$ the equality \eqref{s1} implies
$$
\sum_{ i : y^{(i)} = 1 } \left(\frac{\vrho^{(i)}_k}{ Q} \right)  \;\; \xx^{(i)} = \sum_{ i : y^{(i)} = -1 } \left(\frac{\vrho^{(i)}_k}{Q}\right) \;\;  \xx^{(i)}.
$$
In other words a convex combination of data points of class $1$ equals a convex combination of  data points of class $-1$, a contradiction since the data points are linearly separable.

The theorem then easily follows. If $v_k \neq 0$ for some $k$ then all $\vrho^{(i)}_{k}$ must vanish. But $\mu^{(i)}>0$ (by definition) and $\lambda^{(i,u)}_k>0$ (since $\alpha > 0$). The $\theta^{(u)}$'s are nonnegative and sum to $1$ and so at least one of them is strictly positive, say $\theta^{(u_0)}>0$. By \eqref{s2} it is then clear that   $\vrho^{(i)}_k=0$ for all $i \in [N]$ if and only if $\veps^{(i,u_0)}=0$ for all $i \in [N]$. Recall by definition that
$$
\veps^{(i,u_0)} = \sigma'\left( \;  1 - y^{(i)}  \hat y^{(i)}   \;\right) 
$$
inside the cell $\Omega_{u_0}$, and so $\veps^{(i,u_0)}=0$ for all $i \in [N]$ implies that 
$$
0 = \sigma\left( \;  1 - y^{(i)}  \hat y^{(i)}   \;\right) 
$$
for all $i \in [N]$ as well. Thus $\Loss|_{\Omega_{u_0}}=0$, and since $\bomega\in \overline \Omega_{u_0}$ the continuity of the loss implies $\Loss(\bomega)=0$ as well.
\end{proof}

\begin{theorem*}[Theorem 5  from the paper]
Consider the loss (9) with  $\alpha>0$ and data $\xx^{(i)}, i \in [N]$ that are linearly separable. Assume that the $\mu^{(i)}$ weight both classes equally. Then every local minimum of $\Loss(\bomega)$ is a global minimum.
\end{theorem*}

The proof of theorem 5 relies on the following four auxiliary lemmas.
\begin{lemma} \label{piecewise_polynomial}
Let $\real=I_1 \cup \ldots \cup I_N$ be a partition of the real line into a finite number of non-empty intervals. Let $f(t)$ be a function defined by
 $$
 f(t)= P_j(t) \quad \text{if } t \in I_j,
 $$
 where the $P_1(t), \dots, P_N(t)$ are polynomials. Then there exists  $t_0>0$ such that the function
 $t \mapsto \sign (f(t))$ is constant on $(0,t_0)$. 
\end{lemma}
\begin{proof}
First note that there exists a $t^*$ such that the interval $(0,t^*)$ is contained in one of the intervals $I_j$. On this interval $(0,t^*)$ the function $f(t)$ is simply the polynomial $P_j(t)$. If $P_j(t)$ is the zero polynomial,  then  $\sign (f(t))=0$ for all $t \in (0,t^*)$ and choosing $\tau=t^*$ leads to the claimed result. If $P_j(t)$ is a non-trivial polynomial, it has either no roots on $(0,t^*)$ or a finite number of roots on $(0,t^*)$. In the first case  $\sign (f(t))$ is clearly constant on $(0,t^*)$ and so choosing $\tau=t^*$ gives the claim. In the second case simply choose $\tau$ to be the first root of $P_j(t)$ that is larger than $0$.
\end{proof}

Before turning to the remaining three auxiliary lemmas it is beneficial to recall the decomposition
\begin{gather} \label{tour1}
\Loss |_{\Omega_u}(W,\vv,\bb,c)=
    \phi_0^u(W, \vv)
+ \phi_1^{u}(\bb,\vv)
+ \phi_2^{u}(c)
+ \phi_3^u  \\
 \phi_0^u( W, \vv)= -\sum_k     v_k  \langle  \aaa^{(u)}_k ,  \ww_k  \rangle,\quad
\phi_1^u( \bb,\vv)= -\sum_k     v_k     \alpha^{(u)}_k     b_k,  
 \quad \phi_2^u(c)=      - \gamma^{(u)}   c ,  \label{tour2}
\end{gather}
for the loss on a cell, as well as the constants
\begin{align}
&\aaa_k^{(u)} := \sum_{i} \mu^{(i)}  \varepsilon^{(i,u)}  \lambda_k^{(i,u)} y^{(i)}  \xx^{(i)}    \qquad \alpha_k^{(u)} :=  \sum_{i} \mu^{(i)}   \varepsilon^{(i,u)}  \lambda_k^{(i,u)} y^{(i)} \label{cst1}\\
&\gamma^{(u)} :=\sum_{i} \mu^{(i)}  y^{(i)} \varepsilon^{(i,u)}    \qquad \qquad \qquad \;\, \phi^{u}_{3} :=  \sum_{i} \mu^{(i)} \varepsilon^{(i,u)}. \label{cst2}
\end{align}
used to define the decomposition. By assumption the data $\x^{(i)},i \in [N]$ are linearly separable and so there exists a unit vector $\qq \in \real^{d}$, a bias $\beta \in \real$ and a margin $m>0$ such that the family of inequalities
\begin{align}
\langle \;  \qq \; , \;  y^{(i)} \xx^{(i)} \; \rangle  +  \beta y^{(i)}    \ge m  \label{lin_sep}
\end{align}
hold. Combining \eqref{cst1} with \eqref{lin_sep} gives the estimate
\begin{equation}
\langle  \qq, \aaa_k^{(u)} \rangle +  \beta \alpha_k^{(u)} \ge m   \sum_{i} \mu^{(i)} \varepsilon^{(i,u)}  \lambda_k^{(i,u)},
\label{estimate}
\end{equation}
that will be used repeatedly when proving the remaining auxiliary lemmas.

\begin{lemma}[First perturbation]\label{perturbation1} Let $\bomega=(W,\vv,\bb,c) \in \Omega$ denote any point in the parameter space. Define
$$
\tilde W= \text{ \rm sign}(v_k) \; {\bf e}_k \otimes \qq  \qquad \text{ and } \qquad  \tilde \bb= \beta \,\text{\rm sign}(v_k)\, \ee_k.
$$
For $t \in \R$ let $\bomega(t):=(W+t \tilde W,\vv,\bb+t \tilde\bb,c)$ denote a corresponding perturbation of $\bomega$. Then
\begin{enumerate}[label=(\roman*)]
\item There exists $t_0>0$ and $u \in  \mathcal{I}(\bomega) $ such that $\bomega(t) \in \overline\Omega_u$ for all $t \in [0,t_0).$
\item 
$\Loss(\bomega) \ge  \Loss(\bomega(t))+ t |v_k| m  \sum_{i} \mu^{(i)} \varepsilon^{(i,u)}  \lambda_k^{(i,u)} $ for all $t \in [0,t_0).$
\end{enumerate}
\end{lemma}
\begin{proof} 
To prove (i) let $\bomega(t)=(W+t \tilde W,\vv,\bb+t \tilde\bb,c)=(W(t),\vv, \bb(t), c )$ denote the perturbation considered in the lemma. Without loss of generality, it suffices to consider the case $k=1$. Then the first row of $W(t)$ and the first entry of $\bb(t)$ are given by
$$
\ww_1(t)=\ww_1+ t \text{\rm sign}(v_1) \qq,  \qquad b_1(t) =  b_1 + t \text{\rm sign}(v_1) \beta
$$
whereas the other rows and entries remains unchanged,
$$
\ww_k(t)=\ww_k, \qquad \text{ and } \qquad  b_k(t) =  b_k  \qquad \text{for  }k \ge 2.
$$
Define the constants $A_k^{(i)}$ and $B_k^{(i)}$ as
\begin{align*}
&\langle \ww_1(t),  \xx^{(i)} \rangle  + b_1(t) =  \langle \ww_1,  \xx^{(i)} \rangle  + b_1  + t \text{\rm sign}(v_1) \left( \langle \qq, \xx^{(i)} \rangle + \beta \right) = A_1^{(i)} + B_1^{(i)} t \\
& \langle \ww_k,  \xx^{(i)} \rangle  + b_k = A_k^{(i)}  \qquad \text{for  }k \ge 2,
\end{align*}
 so that the signature  functions can be written as:
\begin{align*}
\sss^{(i,1)}_1(\bomega(t))&= \sign(A_1^{(i)} + B_1^{(i)} t)  \\
\sss^{(i,1)}_k(\bomega(t))&= \sign( A^{(i)}_k)  \qquad \text{for }k \ge 2     \\
\sss^{(i,2)}(\bomega(t))&=
\sign \left[ \; 1 -y^{(i)} \left\{ c+ v_1 \sigma_\alpha( A_1^{(i)} + B_1^{(i)} t))  +       \sum_{k=2}^{K} v_k  \sigma_\alpha ( A_k^{(i)})    \right\}    \right]. 
\label{ccc}
\end{align*}
The functions appearing inside the $\sign$ functions are clearly piecewise defined polynomials, and therefore lemma \ref{piecewise_polynomial} implies that there exists $t_0>0$ such that  $t \mapsto \Sc( \bomega(t))$ is constant on $(0,t_0)$. 
This implies that for $t \in (0,t_0)$,  $\bomega(t)$ either remains in a fixed cell $\Omega_u$ (if none of the entries of $\Sc( \bomega(t))$ are equal to 0) or on the boundary of a fixed cell $\Omega_u$  (if some of the entries of $\Sc( \bomega(t))$ are equal to 0). In both cases we have that  $\bomega(t)\in \overline\Omega_u$ for all $t \in (0,t_0)$. Since $\bomega(t)$ is continuous and since $\overline\Omega_u$ is closed, $\bomega(t)\in \overline{\Omega_u}$ for all $t \in [0,t_0)$ and so (i) holds.  To prove (ii), first note that due to the continuity of the loss, equality \eqref{tour1} holds not only for $\bomega \in \Omega_u$, but also for any $\bomega \in \overline\Omega_u$. By part (i), $\bomega(t)$ remains in some fixed  $\overline \Omega_u$ for all $t$ small enough. Thus (\ref{tour1}-\ref{tour2}) apply. The bilinearity of $\phi^{u}_0$ and $\phi^{u}_1$ then yield
\begin{equation*}
\Loss(\bomega(t))-\Loss(\bomega)=  t  \phi^u_0(\tilde W , \vv) +
t \phi^u_1(\tilde \bb , \vv)  
=  -t|v_k| \big( \langle \aaa_k^{(u)}, \qq \rangle + \alpha_k^{(u)} \beta \big),
\end{equation*}
which combined with  \eqref{estimate} proves (ii). 
\end{proof}

\begin{lemma}[Second perturbation]\label{perturbation2} Let $\bomega=(W, \vv,\bb,c) \in \Omega$ denote a point in the parameter space. Assume that $\vv=\0$ and $c \notin \{-1,+1\}$. Define $$\tilde \vv= \ee_k.$$
For $t \in \R$ let $\bomega(t):=(W,  \vv+t \tilde \vv ,\bb,c)$ denote a corresponding perturbation of $\bomega$. Then
\begin{enumerate}[label=(\roman*)]
\item There exists $t_0>0$ such that $\bomega(t) \in \overline \Omega_u$ for all $u \in \mathcal{I}(\bomega)$ and all  $t \in (-t_0,t_0)$.
\item 
$\Loss(\bomega(t)) - \Loss(\bomega) = - t \left( \langle  \aaa^{(u)}_k ,  \ww_k  \rangle - \alpha_k^u b_k \right)$ for all $t \in (-t_0,t_0)$ and all $u \in \mathcal{I}(\bomega)$.
\end{enumerate}
\end{lemma}
\begin{proof} To prove (i), note that $\vv=\0$ implies the equalities
\begin{align*}
&\sss^{(i,1)}(\bomega(t))= \sign (  W  \xx^{(i)}+\bb) \\
&\sss^{(i,2)}(\bomega(t))= \sign \left[ \; 1 -  c y^{(i)}  -  t y^{(i)}  \tilde \vv^T  \sigma_\alpha (   W \xx^{(i)}+ \bb)      \right]
\end{align*}
for the signature function. But then $1 -  c y^{(i)} \neq 0$ since $c \notin \{ +1, -1\}$, and so there exists an interval $(-t_0,t_0)$ on which all the functions $\sss^{(i,2)}(\bomega(t))$ do not change. Obviously the functions   $\sss^{(i,1)}(\bomega_2(t))$ do not change as well since $W$ and $\bb$ are not perturbed. So the signature $\Sc( \bomega(t))$ does not change on $(-t_0,t_0)$, which yields (i). To prove (ii), choose an arbitrary $u\in \mathcal I(\bomega)$. Since $\bomega(t)$ remains in $\overline\Omega_u$ for all $t \in (-t_0,t_0)$ the relations \eqref{tour1}-\eqref{tour2} imply
 \begin{equation*}
\Loss(\bomega(t))-\Loss(\bomega^*)=  t  \left(  \phi^u_0( W , \tilde \vv) +
 \phi^u_1(\bb , \tilde \vv)   \right)  = - t \left( \langle  \aaa^{(u)}_k ,  \ww_k  \rangle - \alpha_k^u b_k \right)
\end{equation*}
for all $t \in (-t_0,t_0)$, which is the desired result.
 \end{proof}
 
 \begin{lemma}[Third perturbation]\label{perturbation3}  Let $\bomega=(W,\vv,\bb,c) \in \Omega$ denote a point in the parameter space. Assume that $\vv=\0$. Define
 $$
\tilde W=  {\bf e}_k \otimes \qq,  \qquad  \tilde \vv =  \ee_k  \qquad \text{ and } \qquad  \tilde \bb= \beta  \ee_k.
$$
For $t \in \R$ let $\bomega(t):=(W+t \tilde W,\vv +t \tilde \vv,\bb+t \tilde\bb,c)$ denote a corresponding perturbation of $\bomega$. Then
\begin{enumerate}[label=(\roman*)]
\item There exists $t_0>0$ and  $u \in  \mathcal{I}(\bomega) $ such that $\bomega(t) \in \overline\Omega_u$ for all $t \in [0,t_0).$
\item 
$\Loss(\bomega) \ge  \Loss(\bomega(t))+ t \left( \langle \aaa_k^{(u)},\ww_k \rangle + \alpha_k^{(u)} b_k \right) + t^2 m  \sum_{i} \mu^{(i)} \varepsilon^{(i,u)}  \lambda_k^{(i,u)} $ for all $t \in [0,t_0).$
\end{enumerate}
\end{lemma}
 \begin{proof} To prove (i) let $\bomega(t)=(W+t \tilde W,\vv+t\tilde \vv,\bb+t \tilde\bb,c)=(W(t),\vv(t), \bb(t), c )$ denote the perturbation considered in the lemma. Without loss of generality, if suffices to consider the case $k=1$. Define the constants $A_k^{(i)}$ and $B_k^{(i)}$ as
\begin{align*}
& \langle \ww_1(t),  \xx^{(i)} \rangle  + b_1(t) =  \langle \ww_1,  \xx^{(i)} \rangle  + b_1  + t \left( \langle \qq, \xx^{(i)} \rangle + \beta \right) = A_1^{(i)}+B_1^{(i)}t  \\
& \langle \ww_k(t),  \xx^{(i)} \rangle  + b_k(t)  =  \langle \ww_k,  \xx^{(i)} \rangle  + b_k  = A_k^{(i)}  \qquad \text{for  }k \ge 2.
\end{align*}
The fact that $\vv(t)=t \ee_1$ then gives
\begin{align*}
\sss^{(i,1)}_1(\bomega(t))&= \sign(A_1^{(i)}+B_1^{(i)}t  ) \\
\sss^{(i,1)}_k(\bomega(t))&= \sign ( A^{(i)}_k)  \qquad \text{for }k \ge 2     \\
\sss^{(i,2)}(\bomega(t))&=
\sign\left[ \; 1 -y^{(i)} \left\{ c+  t \sigma_\alpha(A_1^{(i)}+B_1^{(i)}t)     \right\}    \right] 
\end{align*}
for the signature functions. As in the proof of part (i) of lemma \ref{perturbation1}, the arguments of the $\sign$ functions are  piecewise defined polynomials and so lemma \ref{piecewise_polynomial} gives the claim. To prove (ii), note that since $\bomega(t)$ remains in a fixed cell $\overline\Omega_u$ for all $t \in (0,t_0)$ the formulas \eqref{tour1}-\eqref{tour2} apply. Expanding the bilinear forms gives
\begin{multline*}
\Loss(\bomega(t)) - \Loss(\bomega)  =  t \Big( \phi_0^u(\tilde W, \vv) +  \phi_0^u( W,  \tilde \vv) + \phi_1^u(\tilde \bb, \vv) +  \phi_1^u( \bb,  \tilde \vv)  \Big) 
+ t^2 \Big( \phi_0^u(\tilde W,  \tilde \vv) +  \phi_1^u( \tilde \bb,  \tilde \vv)  \Big), 
\end{multline*}
and $\vv=\0$ the first order terms are
$$
  \phi_0^u( W,  \tilde \vv) +  \phi_1^u( \bb,  \tilde \vv)  = -  \langle  \aaa^{(u)}_k ,  \ww_k  \rangle - \alpha_k^u b_k .
$$
Applying \eqref{estimate} then yields
$$
\phi_0^u(\tilde W,  \tilde \vv) +  \phi_1^u( \tilde \bb,  \tilde \vv) =  -\langle \aaa_k^{\left(u \right)}, \qq \rangle - \alpha_k^{\left(u \right)} \beta
 \le -  m   \sum_{i} \mu^{(i)} \varepsilon^{(i,u)}  \lambda_k^{(i,u)}  
$$
for second order terms, giving the claim.
\end{proof}

\begin{proof}[Proof of Theorem 5]
 The proof is in two steps. The first step shows that a sub-optimal local minimizer must necessarily be of the form $\bomega=(W,\0,\bb, \pm1),$ while the second step shows that such a sub-optimal minimizer cannot exist if the two classes are equally weighted.

STEP 1: Assume that $\bomega=(W,\vv,\bb,c) \in \Omega$ is a sub-optimal local minimum. Then 
 $\ell^{(i)}(\bomega) >0$ for some data point $\xx^{(i)}$. Take an arbitrary $u\in \mathcal{I}(\bomega)$. By continuity of the loss, there exists $\hat \bomega \in \Omega_u$  such that  $\ell^{(i)}(\hat\bomega) >0$, and, as a consequence, $\varepsilon^{(i,u)}=1$.  Thus
 \begin{equation} \label{newyork}
 \varepsilon^{(i,u)}=1 \qquad \text{for all} \quad u  \in  \mathcal{I}(\bomega).
 \end{equation}
since $u$ was arbitrary. Now choose an arbitrary $k \in [K]$ and  consider the perturbation $\bomega(t):=(W+t \tilde W,\vv,\bb+t \tilde\bb,c)$ described in lemma \ref{perturbation1}.  As $\alpha > 0$ the $\lambda^{(i,u)}_k$ are all strictly positive. By \eqref{newyork}, the term
 $\sum_{i} \mu^{(i)} \varepsilon^{(i,u)}  \lambda_k^{(i,u)}$ appearing in statement  (ii) of lemma \ref{perturbation1}  is strictly positive as well. Since
 $\bomega$ is a local minimum, $v_k$ must necessary be equal to zero, otherwise the considered perturbation would lead to a strict decrease of the loss. Thus $\vv=\0$ since $k$ was arbitrary. Assume that $c \notin \{-1,+1\}$ for the sake of contradiction. The perturbation described in lemma \ref{perturbation2} gives
 \begin{equation} \label{LA}
 \langle \aaa_k^{(u)},\ww_k \rangle + \alpha_k^{(u)} b_k = 0 \qquad \text{for all} \quad u  \in  \mathcal{I}(\bomega),
\end{equation}
which combines with \eqref{newyork}, \eqref{LA} and the perturbation described in lemma \eqref{perturbation3} to give a strict decrease in the loss. This contradicts the fact that $\bomega$ is a local minimum, and so in fact $c \in \{-1,1\}$.

STEP 2. By step 1 a sub-optimal local minimizer must be of the form  $\bomega=(W,\0,\bb, \pm 1)$. Assume $c=1,$ as the argument for the case $c=-1$ is similar. Thus $\bomega=(W,\0,\bb,  1)$. Consider the perturbation
$\bomega(t)=(W,\0,\bb,1-t)$. For $t \in [0,2]$ it then follows that
\begin{align*}
\Loss(\bomega(t))&= \sum_i  \mu^{(i)} \sigma \left( \; 1 -y^{(i)} + y^{(i)}t        \right) \\
&= \sum_{i : y^{(i)}=1} \mu^{(i)}  \sigma ( t ) + \sum_{i: y^{(i)}=-1}  \mu^{(i)}   \sigma ( 2-t) \\
&= \sum_{i : y^{(i)}=1} \mu^{(i)}  t  + \sum_{i: y^{(i)}=-1}  \mu^{(i)}   (2-t) \\
&= 2 \sum_{i: y^{(i)}=-1}  \mu^{(i)}  = 1
\end{align*}
where the equal mass hypothesis
$$
\sum_{i: y^{(i)}=1}  \mu^{(i)} = \sum_{i: y^{(i)}=-1}  \mu^{(i)} = \frac1{2}
$$
justifies the last two equalities. Therefore $\Loss(\bomega)=\Loss(\bomega(t))=1$ for $t$ small enough. But if $t \neq 0$ then $\bomega(t)$ cannot be a local minimizer by stem 1. Thus the point $\bomega=(W,\0,\bb,  1) \in \Omega$ has arbitrarily close neighbors $\bomega(t) \in \Omega$ that have same loss and that are not local minima. This implies that $\bomega$ can not be a local minimum.
\end{proof}

\begin{theorem*}[Theorem 6  from the paper]  
Consider the loss (9) with  $\alpha=0$ and data $\xx^{(i)}, i \in [N] $ that are linearly separable. Assume that $\bomega=(W,\vv,\bb,c)$ is a critical point in the Clarke sense, and that $\xx^{(i)}$ is any data point that contributes a nonzero value to the loss. Then for each hidden neuron $k \in [K]$ either
$$
\text{(i)} \; \langle \ww_k, \xx^{(i)} \rangle + b_k \le 0,\quad \text{or}\quad \text{(ii)}\;v_k=0.
$$
\end{theorem*}

\begin{proof}
The proof of theorem 6 shows
\begin{align*}
\0 &= v_k\left( \sum_{u} \sum_{ i : y^{(i)} = 1 } \theta^{(u)} \mu^{(i)}\veps^{(i,u)} \lambda^{(i,u)}_k\xx^{(i)} - \sum_{u} \sum_{ i : y^{(i)} = -1 } \theta^{(u)} \mu^{(i)}\veps^{(i,u)} \lambda^{(i,u)}_k\xx^{(i)}\right)\\
0 &= v_k\left( \sum_{u} \sum_{ i : y^{(i)} = 1 } \theta^{(u)} \mu^{(i)} \veps^{(i,u)} \lambda^{(i,u)}_k - \sum_{u} \sum_{ i : y^{(i)} = -1 } \theta^{(u)} \mu^{(i)} \veps^{(i,u)} \lambda^{(i,u)}_k\right)
\end{align*}
whenever $\bomega$ is a critical point. If $\ell^{(i)}(\bomega)>0$ for some data point $\xx^{(i)}$ then $\ell^{(i)}>0$ on each neighboring cell $\Omega_u$,  $u \in \mathcal{I}(\bomega)$ by continuity of the loss. This implies that 
$\veps^{(i,u)}=1$ for all $u\in \mathcal{I}(\bomega)$. If $v_k\neq 0$ for some $k$ then
\begin{align*}
\sum_{ i : y^{(i)} = 1 } \vrho^{(i)}_k \xx^{(i)} &= \sum_{ i : y^{(i)} = -1 } \vrho^{(i)}_k \xx^{(i)}\\
\sum_{ i : y^{(i)} = 1 } \vrho^{(i)}_k  &= \sum_{ i : y^{(i)} = -1 } \vrho^{(i)}_k \qquad \text{where} \qquad \vrho^{(i)}_{k} := \sum_{u} \theta^{(u)} \mu^{(i)}\veps^{(i,u)} \lambda^{(i,u)}_k
\end{align*}
and the corresponding  $\vrho^{(i)}_k$ must all vanish since the data $\xx^{(i)},i\in[N]$ are separable. If $\ell^{(i)}>0$ this necessarily implies that $\lambda_k^{(i,u)}=0$ for some $u$ since the $\veps^{(i,u)} = 1$ and at least one $\theta^{(u)}$ does not vanish. This in turn implies
$\sigma( \langle \ww_k, \xx^{(i)} \rangle + b_k)=0$ due to the definition of the $\lambda_k^{(i)}$.
\end{proof}

\section*{Proofs of theorems from  section 4: Exact Penalties and Multi-Class Structure}

The proof of theorem 7 requires modifying the notion of a cell. This modification is straightforward; it simply accounts for the fact that the penalized loss
\begin{equation}\label{eq:pen2}
\E_{\gamma}\big(\bomega^{(1)},\ldots,\bomega^{(R)}\big) := \sum^{R}_{r=1} \Loss^{(r)}\big(\bomega^{(r)}\big) + \gamma \mathcal{R}\big(\breve \bomega^{(1)},\ldots,\breve \bomega^{(R)}\big)
\end{equation}
has the $R$-fold Cartesian product $\Omega \times \cdots \times \Omega$ as its parameter domain. The notion of a cell $\Omega_u$ for the model \eqref{eq:pen2} consists of sets (Cartesian products) of the form
\begin{equation}\label{eq:prodset}
\Omega_{u} = \Omega_{u^{(1)}} \times \Omega_{u^{(2)}} \times \cdots \times \Omega_{ u^{(R)} },
\end{equation}
where each $u^{(r)} \in \{-1,1\}^{ND}$ denotes a signature for the individual two-class losses. Thus a binary vector of the form 
$$u = \big( u^{(1)},\ldots,u^{(R)}\big) \in \{-1,1\}^{NDR} \qquad u^{(r)} \in \{-1,1\}^{ND}$$
defines a signature for the full model. That sets of the form \eqref{eq:prodset} cover the product space $\Omega \times \cdots \times \Omega$ up to a set of measure zero follows easily from the fact that if $\big( \bomega^{(1)},\ldots,\bomega^{(R)} \big) \notin \Omega_u$ for all $u$ then at least one of the $u^{(r)}$ (say $u^{(1)}$ WLOG) lies in the set
$$
\mathcal{N} := \Omega \setminus \left( \bigcup_{ u^{(1)} \in \{0,1\}^{ND} } \Omega_{u^{(1)}} \right)
$$
which has measure zero in $\Omega$. Thus $u$ must lie in the set
$$
\mathcal{N} \times \Omega \times \cdots \times \Omega
$$
which has measure zero in the product space $\Omega \times \cdots \times \Omega$, and so the union of the $R$ measure zero sets of the form
$$
\Omega \times \cdots \times \mathcal{N} \times \cdots \times \Omega
$$
contains all parameters $\big( \bomega^{(1)},\ldots,\bomega^{(R)} \big)$ that do not lie in a cell. The proof also relies following auxiliary lemma.
\begin{lemma}\label{lem:yyy}
For any $R$ vectors $\xx^{(1)},\ldots,\xx^{(R)} \in \R^{d},$ if
$$
\|\xx^{(r)}\|^2 = \frac1{R-1} \sum_{s \neq r} \langle \xx^{(s)} , \xx^{(r)} \rangle \qquad \text{for all} \qquad r \in \{1,\ldots,R\}
$$
then $\xx_1 = \cdots = \xx_R$.
\end{lemma}
\begin{proof}
By relabelling if necessary, it suffices to assume $\xx^{(1)}$ has largest norm amongst the $\xx^{(r)}$. Thus $\| \xx^{(1)}\| \geq \| \xx^{(r)} \|$ for all $1 \leq r \leq R$. If $\|\xx^{(1)}\| = 0$ then there is nothing to prove. Otherwise apply Cauchy-Schwarz and the hypothesis of the lemma to find
\begin{align*}
 \|\xx^{(1)}\|^2 &\leq \frac1{R-1} \sum_{s \neq 1} \|\xx^{(s)}\| \| \|\xx^{(1)}\| \\
\|\xx^{(1)}\| &\leq \frac1{R-1} \sum_{s \neq 1} \|\xx^{(s)}\|.
\end{align*}
The latter inequality implies $\|\xx^{(1)}\| = \cdots = \| \xx^{(R)}\|$ since $\xx^{(1)}$ has largest norm. Thus
\begin{align*}
\| \xx^{(1)}\|^{2} &= \frac1{R-1} \sum_{ s \neq 1} \cos \theta_r \|\xx^{(1)}\|^2\\
1 &= \frac1{R-1} \sum_{ s \neq 1} \cos \theta_r
\end{align*}
by the hypothesis of the lemma. The latter equality implies $\cos \theta_r = 1$ for all $r,$ and so the lemma is proved.
\end{proof}

\begin{theorem*}[Theorem 7  from the paper]\label{thm:multiappx}
If $\gamma > 0$ then the following hold for \eqref{eq:pen2} ---
\begin{enumerate}[label=(\roman*)]
\item The penalty is exact, that is, at \textbf{any} critical point $\big( \bomega^{(1)},\ldots,\bomega^{(R)} \big)$ of $\E_{\gamma}$ the equalities
\begin{align*}
\omega^{(\ell,1)} &= \cdots = \omega^{(\ell,R)} = \bar \omega^{(\ell)} := \frac1{R} \sum^{R}_{r=1} \omega^{(\ell,r)}\\
\bb^{(\ell,1)} &= \cdots = \bb^{(\ell,R)} = \bar \bb^{(\ell)} := \frac1{R} \sum^{R}_{r=1} \bb^{(\ell,r)}
\end{align*}
hold for all $\ell \in [L]$.
\item At \textbf{any} critical point of $\E_{\gamma}$ the two-class critical point relations
\begin{equation}\label{eq:magic2}
\0 \in \partial_{0} \Loss^{(r)}( \breve \bomega , \vv_r , c_r )
\end{equation}
hold for all $r \in [R]$.
\end{enumerate}
\end{theorem*}
\begin{proof}
Let $\big( \bomega^{(1)},\ldots,\bomega^{(R)} \big)$ denote any critical point. For each $(\ell,r)$ the equalities
\begin{align*}
\nabla_{\omega^{(\ell,r)}} \mathcal{R} = \gamma\big( \omega^{(\ell,r)} - \tilde \omega^{(\ell,r)} \big) \qquad  \tilde \omega^{(\ell,r)} := \frac1{R-1} \sum_{s \neq r}\omega^{(\ell,s)} \\
\nabla_{\bb^{(\ell,r)}} \mathcal{R} = \gamma\big( \bb^{(\ell,r)} - \tilde \bb^{(\ell,r)} \big) \qquad  \tilde \bb^{(\ell,r)} := \frac1{R-1} \sum_{s \neq r}\bb^{(\ell,s)} 
\end{align*}
follow from a straightforward calculation. By definition of a critical point, for each cell $\Omega_u$ adjacent to the critical point there exist corresponding constants $\theta^{(u)}\geq 0$ with $\sum_{u} \theta^{(u)} = 1$ so that the equalities
\begin{align}\label{eq:cprel}
\0 &= \sum_{u} \theta^{(u)} \nabla_{\vv_r} \bar \Loss|_{\Omega_u} \nonumber \\
\0 &= \sum_{ u } \theta^{(u)} \left( \nabla_{\omega^{(\ell,r)}} \bar \Loss|_{\Omega_u}  + \nabla_{\omega^{(\ell,r)}} \mathcal{R} \right)  =  \gamma\big( \omega^{(\ell,r)} - \tilde \omega^{(\ell,r)} \big) + \sum_{ u } \theta^{(u)}  \nabla_{\omega^{(\ell,r)}} \bar \Loss|_{\Omega_u} \nonumber \\
\0 &= \sum_{ u } \theta^{(u)} \left( \nabla_{\bb^{(\ell,r)}} \bar \Loss|_{\Omega_u}  + \nabla_{\bb^{(\ell,r)}} \mathcal{R} \right)  =  \gamma\big( \bb^{(\ell,r)} - \tilde \bb^{(\ell,r)} \big) + \sum_{ u } \theta^{(u)} \nabla_{\bb^{(\ell,r)}} \bar \Loss|_{\Omega_u} 
\end{align}
hold for all $\ell \in [L]$ and $r \in [R]$, where the final equalities in the second and third line follow from the fact that $\mathcal{R}$ is smooth and so its gradients do not depend upon the cell. Now on any cell $\Omega_u$ the loss $\Loss^{(r)}$ decomposes into a sum of multilinear forms
\begin{align*}
\Loss^{(r)}|_{\Omega_u} &= \phi_0^{(u,r)}\big( \omega^{(1,r)},\ldots,\omega^{(L,r)},\vv_r \big) + \sum^{L-1}_{\ell=1} \phi_{\ell}^{(u,r)}\big( \bb^{(\ell,r)} , \omega^{(\ell+1,r)},\ldots,\omega^{(L,r)},\vv_r\big)\\ &+ \phi_{L}^{(u,r)}\big(\bb^{(L,r)},\vv_r\big) + \phi_{L+1}^{(u,r)}(c_r) + \phi_{L+2}^{(u,r)}
\end{align*}
by theorem 1. For any multilinear form $\phi(\vv_1,\ldots,\vv_n)$ the equality
$$
\phi(\vv_1,\ldots,\vv_n) = \langle \vv_k , \nabla_{\vv_k} \phi (\vv_1,\ldots,\vv_n) \rangle
$$
holds for all $k \in [n]$ by Euler's theorem for homogeneous functions. Taking the inner-product of \eqref{eq:cprel} with $\vv_r,\omega^{(L,r)}$ and $\bb^{(L,r)}$ then shows
\begin{align}\label{eq:linforms}
0 &= \sum_{u} \theta^{(u)}\big( \phi^{(u,r)}_{0} + \cdots + \phi^{(u,r)}_{L} \big) \nonumber \\
0 &= \sum_{u} \theta^{(u)}\big( \phi^{(u,r)}_{0} + \cdots + \phi^{(u,r)}_{L-1} \big) + \gamma \big( \| \omega^{(L,r)} \|^{2} - \langle \omega^{(L,r)} , \tilde \omega^{(L,r)} \rangle \big) \nonumber \\
0 &= \sum_{u} \theta^{(u)}\big( \phi^{(u,r)}_{L} \big) + \gamma \big( \| \bb^{(L,r)} \|^{2} - \langle \bb^{(L,r)} , \tilde \bb^{(L,r)} \rangle \big)
\end{align}
which upon adding the second and third equalities yields
$$
\| \omega^{(L,r)}\|^{2} + \|\bb^{(L,r)}\|^{2} = \langle \omega^{(L,r)} , \tilde \omega^{(L,r)} \rangle + \langle \bb^{(L,r)} , \tilde \bb^{(L,r)} \rangle
$$
for all $r \in [R]$. By the definitions of $\tilde \omega^{(L,r)}$ and $\tilde \bb^{(L,r)}$ (c.f. lemma \ref{lem:yyy}), this can happen if and only if
$$
\omega^{(L,1)} = \cdots = \omega^{(L,R)} \qquad \text{and} \qquad \bb^{(L,1)} = \cdots = \bb^{(L,R)}.
$$
Using this in the second and third equations in \eqref{eq:linforms} then shows that
\begin{equation}\label{eq:vanishing}
0 = \sum_{u} \theta^{(u)}\big( \phi^{(u,r)}_{0} + \cdots + \phi^{(u,r)}_{L-1} \big) = \sum_{u} \theta^{(u)}\big( \phi^{(u,r)}_{L} \big)
\end{equation}
for all $r \in [R]$ as well. Now take the inner-product of \eqref{eq:cprel} with $\omega^{(L-1,r)}$ and $\bb^{(L-1,r)}$ to find
\begin{align*}
0 &= \sum_{u} \theta^{(u)}\big( \phi^{(u,r)}_{0} + \cdots + \phi^{(u,r)}_{L-2} \big) + \gamma \big( \| \omega^{(L-1,r)} \|^{2} - \langle \omega^{(L-1,r)} , \tilde \omega^{(L-1,r)} \rangle \big) \nonumber \\
0 &= \sum_{u} \theta^{(u)}\big( \phi^{(u,r)}_{L-1} \big) + \gamma \big( \| \bb^{(L-1,r)} \|^{2} - \langle \bb^{(L-1,r)} , \tilde \bb^{(L-1,r)} \rangle \big)
\end{align*}
Adding these equations and using \eqref{eq:vanishing} then reveals
$$
\omega^{(L-1,1)} = \cdots = \omega^{(L-1,R)} \qquad \text{and} \qquad \bb^{(L-1,1)} = \cdots = \bb^{(L-1,R)}
$$
must hold as well, and so also
$$
0 = \sum_{u} \theta^{(u)}\big( \phi^{(u,r)}_{0} + \cdots + \phi^{(u,r)}_{L-2} \big) = \sum_{u} \theta^{(u)}\big( \phi^{(u,r)}_{L-1} \big)
$$
must hold. Continuing from $\ell = L-2$ to $\ell = 1$ by induction reveals
$$
\omega^{(\ell,1)} = \cdots = \omega^{(\ell,R)} \qquad \text{and} \qquad \bb^{(\ell,1)} = \cdots = \bb^{(\ell,R)}.
$$
for all $\ell \in [L],$ and so part (i) is proved. Part (ii) then follows from part (i) since the equalities
$$
\tilde \omega^{(\ell,r)} = \bar \omega^{(\ell)} = \omega^{(\ell,r)} \qquad \tilde \bb^{(\ell,r)} = \bb^{(\ell)} = \bb^{(\ell,r)}
$$
hold for all $(\ell,r)$ at any critical point. Thus \eqref{eq:cprel} yields
\begin{align}\label{eq:finalrel}
\0 &= \sum_{u} \theta^{(u)} \nabla_{\vv_r}  \Loss^{(r)}|_{\Omega_{u^{(r)}}} \nonumber  \\
\0 &= \sum_{ u } \theta^{(u)} \nabla_{\omega^{(\ell,r)}}  \Loss^{(r)}|_{\Omega_{u^{(r)}}} \nonumber \\
\0 &= \sum_{ u } \theta^{(u)} \nabla_{\bb^{(\ell,r)}}  \Loss^{(r)}|_{\Omega_{u^{(r)}}} 
\end{align}
for all $\ell \in [L], r \in [R]$. Now consider \eqref{eq:finalrel} for $r=1$. Any cells appearing in the sum \eqref{eq:finalrel} satisfy either $(\breve \bomega,\vv_1,c_1) \in \Omega_{u^{(1)}}$ or $(\breve \bomega,\vv_1,c_1) \in \partial \Omega_{u^{(1)}}$. If $(\breve \bomega,\vv_1,c_1) \in \Omega_{u^{(1)}}$ for some $u^{(1)}$ then \eqref{eq:finalrel} must consist only of gradients on the single cell $\Omega_{u^{(1)}}$ and so $(\breve \bomega,\vv_1,c_1) \in \Omega_{u^{(1)}}$ is a critical point of $\Loss^{(1)}$ in the classical sense. If $(\breve \bomega,\vv_1,c_1) \in \partial \Omega_{u^{(1)}}$ for some $u^{(1)}$ in the sum then $(\breve \bomega,\vv_1,c_1) \in \partial \Omega_{u^{(1)}}$ for all cells $u$ the sum. Thus \eqref{eq:finalrel} consists of a positive combination of gradients of $\Loss^{(1)}$ on cells adjacent to $(\breve \bomega,\vv_1,c_1),$ and so $(\breve \bomega,\vv_1,c_1)$ defines a critical point of $\Loss^{(1)}$ in the extended Clarke sense. Applying this reasoning for $r=2,\ldots,R$ then yields part (ii) and proves the theorem.
\end{proof}

The following preliminary lemma will aid the proofs of the stated corollaries to this theorem.
\begin{lemma}\label{lem:locminlem}
Consider a piecewise multilinear loss
$$
\bar \Loss( \bomega ) = \sum^{R}_{r=1} \bar \Loss^{(r)}(\breve \bomega , \vv_r , c_r )
$$
\end{lemma}
that satisfies theorem 7, and for $\gamma > 0$ let
$$
\E_{\gamma}(\bomega^{(1)},\ldots,\bomega^{(R)}) := \sum^{R}_{r=1} \bar \Loss^{(r)}\big( \bomega^{(r)} \big) + \gamma \mathcal{R}\big( \breve \bomega^{(1)},\ldots,\breve \bomega^{(R)}\big) \qquad \bomega^{(r)} = (\breve \bomega^{(r)},\vv_r  , c_r )
$$
denote its corresponding exact penalization. If $\big(\bomega^{(1)},\ldots,\bomega^{(R)}\big)$ is a local minimum of $\E_{\gamma}$ and $\breve \bomega := (\breve \bomega^{(1)}+\cdots+\breve \bomega^{(R)})/R$ then $\bomega := (\breve \bomega, \vv_1,c_1,\ldots,\vv_R,c_R)$ is a local minimum of $\bar \Loss$.
\begin{proof}
As $\big(\bomega^{(1)},\ldots,\bomega^{(R)}\big)$ is a local minimum of $\E_{\gamma}$ it must be a critical point. Thus each $\bomega^{(r)}$ takes the form
$$
\bomega^{(r)} = (\breve \bomega , \vv_r , c_r)
$$
by theorem 7. If $\breve \bomega$ is not a local minimizer of $\bar \Loss$ then there exists a sequence $\bomega_k \to \bomega$ for which $\bar \Loss( \bomega_{k} ) < \bar \Loss ( \bomega)$ holds. Define the identically replicated points 
$$ \bomega^{(r,k)} := (\breve \bomega_k,\vv^{(k)}_r,c^k_r), \quad \text{so that} \quad \big(\bomega^{(1,k)},\ldots,\bomega^{(R,k)}\big) \to \big(\bomega^{(1)},\ldots,\bomega^{(R)}\big)$$ and moreover
$$
\E_{\gamma}\big(\bomega^{(1,k)},\ldots,\bomega^{(R,k)}\big) = \bar \Loss( \bomega _k ) < \bar \Loss( \bomega ) = \E_{\gamma}\big(\bomega^{(1)},\ldots,\bomega^{(R)}\big)
$$
which contradicts the fact that $\big(\bomega^{(1)},\ldots,\bomega^{(R)}\big)$ is a local minimizer of $\E_{\gamma}$.
\end{proof}

For our multiclass analysis we begin at $\alpha=1$ and study the deep linear problem
\begin{align}\label{eq:dlmulti}
\Loss(\bomega) &= \sum^{R}_{r=1} \Loss^{(r)}(\bomega \big) \quad \text{for} \\
 \Loss^{(r)}(\bomega) &:= \sum\mu^{(i,r)} \sigma\big( 1 - y^{(i,r)}( \langle \vv_r , \xx^{(i,L)} \rangle + c_r) \big) \nonumber
\end{align}
using the soft penalty approach. The features $\xx^{(i,L)} := W^{(L)}\cdots W^{(1)}\xx^{(i)}$ result from a deep linear network, and so if we define $\bar \vv^{(r)} := (W^{(L)}\cdots W^{(1)})^{T}\vv_r$ then we may once again view \eqref{eq:dlmulti} as a convex loss 
$$E^{(1)}(\bar \vv_1,c_1) + \cdots + E^{(R)}(\bar \vv_R,c_R)$$
with over-parametrized arguments. If the positive weights $\mu^{(i,r)}>0$ satisfy
$
\sum_{y^{(i,r)}=1} \mu^{(i,r)}  = \sum_{ y^{(i,r)}=-1} \mu^{(i,r)} = \frac1{2}
$
then we say that the $\mu^{(i,r)}$ give equal weight to all classes. Directly appealing to the critical point relations \eqref{eq:magic} gives our first simple corollary using this approach.

\begin{corollary}[Multiclass Deep Linear Networks, I]
Consider the loss  \eqref{eq:dlmulti} and its corresponding penalty \eqref{eq:pen2} with $\gamma >0$ and arbitrary data. Assume that $\bomega = (\bomega^{(1)},\ldots,\bomega^{(R)})$ is any critical point of $\E_{\gamma}$ in the Clarke sense. If $\bar \vv^{(r)} \neq \0$ for all $r \in [R]$ then $\bomega$ is a global minimum of $\Loss$ and of $\E_{\gamma}$.
\end{corollary}
\begin{proof}
By theorem 7 any critical point of $\E_{\gamma}$ yields a common set of weights $W^{(\ell)}, \ell \in [L]$ for which the two-class critical point relations
$$
\0 \in \partial_{0} \Loss^{(r)}\big(W^{(1)},\ldots,W^{(L)},\vv_r,c_r)
$$
hold for all $r \in [R]$. By theorem 100, if $\bar \vv_r \neq \0$ then $(\bar \vv_r,c_r)$ is a global minimum of the convex function
$$
E^{(r)}(\ww_r,c_r) := \sum^{N}_{i=1} \mu^{(i,r)} \sigma\big( 1 - y^{(i,r)}( \langle \ww_r , \xx^{(i)} \rangle + c_r) \big)
$$
and so $\0 \in \partial E^{(r)}(\bar \vv_r,c_r)$ by definition of the subgradient for convex functions. If $\bar \vv_r \neq \0$ for all $r \in [R]$ it therefore follows that
$$
\0 \in \partial E^{(r)}(\bar \vv_r , c_r)
$$
for all $r \in [R]$. Finally, define the convex function $E(\ww_1,c_1,\ldots,\ww_R,c_R) := E^{(1)}(\ww_1,c_1)+\cdots+E^{(R)}(\ww_R,c_R)$ and note that the sum rule
$$
\partial E( \bar \vv_1 , c_1 ,\ldots,\bar \vv_R , c_R) = \sum^{R}_{r=1} \partial E^{(r)}(\bar \vv_r , c_r)
$$
holds since each $E^{(r)}$ is Lipschitz. Thus $\0 \in \partial E( \bar \vv_1 , c_1 ,\ldots,\bar \vv_R , c_R),$ and as $\gamma >0$ it follows that $\bomega$ is a global minimizer.
\end{proof}

\begin{corollary}[Multiclass Deep Linear Networks, II]
Consider the loss  \eqref{eq:dlmulti} and its corresponding penalty \eqref{eq:pen2} with $\gamma >0$ and arbitrary data. Assume that $\bomega = (\bomega^{(1)},\ldots,\bomega^{(R)})$ is a local minimum of $\E_{\gamma}$ with $\bar \vv^{(r)} = \0$ for some $r \in [R]$. If the $\mu^{(i,r)}$ give equal weight to all classes then $\bomega$ is a global minimum of $\Loss$ and of $\E_{\gamma}$.
\end{corollary}
\begin{proof}
Any local minimum $(\bomega^{(1)},\ldots,\bomega^{(R)})$ of $\E_{\gamma}$ is necessarily a critical point, and so each $\bomega^{(r)}$ takes the form
$$
\bomega^{(r)} = (W^{(1)},\ldots,W^{(L)},\vv_r,c_r)
$$
for some common weight matrices $W^{(\ell)}, \ell \in [L]$. Moreover, $\0 \in \partial_{0} \Loss^{(r)}(W^{(1)},\ldots,W^{(L)},\vv_r,c_r)$ for all $r \in [R]$ as well.

Consider any $r \in [R]$ for which $\bar \vv^{(r)} = \0$. Then as a function of $c_r$ the convex function $E^{(r)}$ obeys
$$
E^{(r)}(\0,c_r) = \sum^{N}_{i=1} \mu^{(i,r)}\sigma\big( 1 - y^{(i,r)}c_r \big) = \frac{ \sigma( 1 - c_r ) + \sigma(1+c_r) }{2}
$$
due to the equal weight hypothesis. Thus $\E_{\gamma}$ can only attain a local minimum if $c_r$ lies in the unit interval $-1 \leq c_r \leq 1,$ and moreover $E(\0,c_r) \equiv 1$ is constant on the unit interval. It therefore suffices to assume that $-1 < c_r < 1$ without loss of generality, and so in particular, that the function $E^{(r)}$ is differentiable (in fact smooth) near $\bomega$. Define the perturbation $\tilde \bomega$ of $\bomega = (\bomega^{(1)},\ldots,\bomega^{(R)})$ as
\begin{align*}
\tilde W^{(\ell,r)} &:= W^{(\ell)} + \delta^{(\ell)}X^{(\ell)} \qquad \tilde W^{(\ell,s)} := W^{(\ell)} \quad (s \neq r),\\
\tilde \vv_r &:= \vv_r + \delta^{(0)}\ww_r,
\end{align*}
where the $\delta^{(\ell)},\ell \in [L]$ and $\delta^{(0)}$ are small scalars, the $X^{(\ell)}$ are arbitrary matrices and $\ww_r$ is an arbitrary vector. Then the energy $\E_{\gamma}$ becomes
$$
\E_{\gamma}( \tilde \bomega ) - \E_{\gamma}( \bomega ) = E^{(r)}\big( (\tilde W^{(L,r)}\cdots \tilde W^{(1,r)})^{T} \tilde \vv_r , c_r \big) - E^{(r)}(\0,c_r) + \gamma \sum^{L}_{\ell=1} \big(\delta^{(\ell)}\big)^2\|X^{(\ell)}\|^{2}.
$$
Define the vector
$$
\zz_r := \sum^{N}_{i=1} \mu^{(i,r)}y^{(i,r)}\xx^{(i)},
$$
and note that
$$
E^{(r)}\big( (\tilde W^{(L,r)}\cdots \tilde W^{(1,r)})^{T} \tilde \vv_r , c_r \big) - E^{(r)}(\0,c_r) = -\langle (\tilde W^{(L,r)}\cdots \tilde W^{(1,r)})\zz_r , \tilde \vv_r \rangle
$$
for all $\delta^{(0)},\delta^{(\ell)}$ sufficiently small. As $\bomega$ is a local minimizer, the inequality
\begin{equation}\label{eq:thisstar}
\gamma \sum^{L}_{\ell=1} \big(\delta^{(\ell)}\big)^2\|X^{(\ell)}\|^{2} \geq \langle (\tilde W^{(L,r)}\cdots \tilde W^{(1,r)})\zz_r , \tilde \vv_r \rangle
\end{equation}
must therefore hold for all $X^{(\ell)},\ww_r$ and corresponding $\delta^{(\ell)},\delta^{(0)}$ sufficiently small. 

First, apply \eqref{eq:thisstar} with $\delta^{(\ell)} = 0$ for all $\ell \in [L]$ to find that
$$
0 \geq \langle W^{(L)}\cdots W^{(1)} \zz_r , \vv_r + \delta^{(0)}\ww_r \rangle
$$
for any $\ww_r$ arbitrary and corresponding $\delta^{(0)}$ sufficiently small. But $\bar \vv^{(r)} = \0$ and so $0 = \langle W^{(L)}\cdots W^{(1)} \zz_r , \ww_r \rangle$ for all $\ww_r$, whence
$$
W^{(L)}\cdots W^{(1)} \zz_r = \0.
$$
Now, if $\zz_r = \0$ then
$$
\0 = \nabla E^{(r)}(\0,c_r),
$$
and so if $\zz_r$ for all $r \in [R]$ for which $\bar \vv^{(r)} = \0$ then $\0 = \nabla E^{(r)}(\0,c_r)$ for all such $r \in [R]$. If $\bar \vv_r \neq \0$ the relation $\0 \in  \partial E^{(r)}(\bar \vv_r,c_r)$ holds as well (as in the previous corollary), and so $\bomega$ is a global minimum.

It therefore remains to consider the case where there exists $r \in [R]$ for which $\bar \vv^{(r)} = \0$ but $\zz_r \neq \0$. As $W^{(L)}\cdots W^{(1)} \zz_r = \0$ there exists an index $0 \leq k \leq L-1$ for which
$$
W^{(k)}\cdots W^{(1)}\zz_r \neq \0 \qquad W^{(k+1)}\cdots W^{(1)}\zz_r = \0,
$$
where clearly $k = 0$ means $\zz_r \neq \0$ but $W^{(1)}\zz_r = \0$. Apply \eqref{eq:thisstar} with $\delta^{(\ell)} = 0$ for all $\ell \neq k+1$ to find that
$$
\gamma \big(\delta^{(k+1)}\big)^{2}\| X^{(k+1)} \|^{2} \geq \langle W^{(L)}\cdots W^{(k+2)}( W^{(k+1)} + \delta^{(k+1)}X^{(k+1)})W^{(k)}\cdots W^{(1)}\zz_r , \vv_r + \delta^{(0)} \ww_r \rangle
$$
for all $X^{(k+1)},\ww_r$ arbitrary and corresponding $\delta^{(k+1)},\delta^{(0)}$ sufficiently small. But $W^{(k+1)}\cdots W^{(1)}\zz_r = \0$ and so 
\begin{equation}\label{eq:nextstar}
\gamma \big(\delta^{(k+1)}\big)^{2}\| X^{(k+1)} \|^{2} \geq \delta^{(k+1)}\langle W^{(L)}\cdots W^{(k+2)}X^{(k+1)}W^{(k)}\cdots W^{(1)}\zz_r , \vv_r + \delta^{(0)} \ww_r \rangle
\end{equation}
must hold for all $X^{(k+1)},\ww_r$ arbitrary and corresponding $\delta^{(k+1)},\delta^{(0)}$ sufficiently small as well. Now apply \eqref{eq:nextstar} with $\delta^{(0)} = 0$ to see
$$
\gamma \big(\delta^{(k+1)}\big)^{2}\| X^{(k+1)} \|^{2} \geq \delta^{(k+1)}\langle W^{(L)}\cdots W^{(k+2)}X^{(k+1)}W^{(k)}\cdots W^{(1)}\zz_r , \vv_r \rangle,
$$
but as $W^{(k)}\cdots W^{(1)}\zz_r \neq \0$ and $X^{(k+1)}$ is arbitrary this can happen if and only if $(W^{(L)}\cdots W^{(k+2)})^{T}\vv_r = \0$. But then \eqref{eq:nextstar} shows
$$
\gamma \big(\delta^{(k+1)}\big)^{2}\| X^{(k+1)} \|^{2} \geq \delta^{(k+1)}\delta^{(0)}\langle W^{(L)}\cdots W^{(k+2)}X^{(k+1)}W^{(k)}\cdots W^{(1)}\zz_r , \ww_r \rangle
$$
must hold for all $X^{(k+1)},\ww_r$ and corresponding $\delta^{(k+1)},\delta^{(0)}$ sufficiently small. Take $\delta^{(k+1)} = \big(\delta^{(0)}\big)^2$ for $\delta^{(0)}$ small to find
$$
\gamma \big(\delta^{(0)}\big)^{4}\| X^{(k+1)} \|^{2} \geq \big( \delta^{(0)} \big)^{3}\langle W^{(L)}\cdots W^{(k+2)}X^{(k+1)}W^{(k)}\cdots W^{(1)}\zz_r , \ww_r \rangle
$$
and so in fact $\langle W^{(L)}\cdots W^{(k+2)}X^{(k+1)}W^{(k)}\cdots W^{(1)}\zz_r , \ww_r \rangle = 0$ for $X^{(k+1)},\ww_r$ arbitrary. This cannot happen unless $k \leq L-2$ and $\0 = W^{(L)}\cdots W^{(k+2)},$ in which case in fact $\bar \vv^{(r)} = \0$ for \textbf{all} $r \in [R]$. But then $E^{(r)}$ is differentiable for \textbf{all} $r \in [R]$ near $(\0,c_r)$, and so $(W^{(1)},\ldots,W^{(L)},\vv_1,\ldots,\vv_R,c_1,\ldots,c_R)$ is a differentiable local minimum of $\bar \Loss$ by lemma \ref{lem:locminlem}. As $\bar \Loss$ is piecewise multilinear and the minimum is differentiable, $\bar \Loss$ must be constant near the minimum. Take $X^{(\ell)},\ww_r$ arbitrary and define
$$
\tilde W^{(\ell)} = W^{(\ell)} + \delta X^{(\ell)} \qquad \tilde \vv_r = \vv_r + \delta \ww_r
$$
for all $\delta$ sufficiently small. Then
$$
\bar \Loss\big( \tilde W^{(1)},\ldots,\tilde W^{(L)} , \tilde V \big) - \bar \Loss\big( W^{(1)},\ldots,W^{(R)},V\big) = 0
$$
for all $\delta$ small enough, and since
$$
\bar \Loss\big( \tilde W^{(1)},\ldots,\tilde W^{(L)} , \tilde V \big) - \bar \Loss\big( W^{(1)},\ldots,W^{(R)},V\big) = \langle \big( \tilde W^{(L)} \cdots \tilde W^{(1)} \big) ^{T} \tilde V ,  Z \rangle \qquad Z = [\zz_1,\ldots,\zz_R]
$$
it follows that $\langle \big( \tilde W^{(L)} \cdots \tilde W^{(1)} \big) ^{T} \tilde V ,  Z \rangle = 0$ for all $\delta$ small enough. Expanding in powers of $\delta$ yields
$$
0 = f_0 + \delta f_1\big(X^{(1)},\ldots,X^{(L)},W\big) + \cdots + \delta^{(L+1)} f_{L+1}\big(X^{(1)},\ldots,X^{(L)},W\big).
$$
for some constant $f_0$ and functions $f_{\ell}, \ell \in [L+1]$ and all $\delta$ small enough. But then
$$
f_{L+1}\big(X^{(1)},\ldots,X^{(L)},W\big) = 0
$$
for all $X^{(\ell)},W$ arbitrary. As $f_{L+1}\big(X^{(1)},\ldots,X^{(L)},W\big) = \langle \big( X^{(L)} \cdots \tilde X^{(1)} \big) ^{T}  W ,  Z \rangle $ this can happen if and only if $Z = 0$. Thus 
$$
\zz_r = \nabla E^{(r)}(\0,c_r) = \0
$$
for all $r \in [R]$ and so $\bomega$ is a global minimum.
\end{proof}

To finish the analysis, we first recall the loss
\begin{align}\label{eq:lrmulti}
\Loss(\bomega) &= \sum^{R}_{r=1} \Loss^{(r)}(\bomega \big) \quad \text{for} \\
 \Loss^{(r)}(\bomega) &:= \sum\mu^{(i,r)} \sigma\big( 1 - y^{(i,r)}( \langle \vv_r , \xx^{(i,1)} \rangle + c_r) \big) \nonumber
\end{align}
for a leaky network with one hidden layer.
\begin{corollary*}[Corollary 3 from the paper]
Consider the loss  \eqref{eq:lrmulti} and its corresponding penalty \eqref{eq:pen2} with $\gamma >0, 0 < \alpha < 1$ and data $\x^{(i)}, i \in [N]$ that are linearly separable. 
\begin{enumerate}[label=(\roman*)]
\item Assume that $\bomega = (\bomega^{(1)},\ldots,\bomega^{(R)})$ is a critical point of $\E_{\gamma}$ in the Clarke sense. If $\vv^{(r)} \neq \0$ for all $r \in [R]$ then $\bomega$ is a global minimum of $\Loss$ and of $\E_{\gamma}$.
\item Assume that the $\mu^{(i,r)}$ give equal weight to all classes. If $\bomega = (\bomega^{(1)},\ldots,\bomega^{(R)})$ is a local minimum of $\E_{\gamma}$ and $\vv_r = \0$ for some $r \in [R]$ then $\bomega$ is a global minimum of $\Loss$ and of $\E_{\gamma}$.
\end{enumerate}
\end{corollary*}
\begin{proof}
For part (i), note once again that theorem 7 any critical point of $\E_{\gamma}$ yields a common set of weights $(W,\bb)$ for which the two-class critical point relations
$$
\0 \in \partial_{0} \Loss^{(r)}\big(W,\bb,\vv_r,c_r)
$$
hold for all $r \in [R]$. By theorem 4, if $\vv_r \neq \0$ for all $r \in [R]$ then $(W,\bb,\vv_r,c_r)$ is a global minimum of $\Loss^{(r)}$ for all $r \in [R]$. But the $\x^{(i)}$ are separable, and so
$$
0 = \Loss^{(r)}(W,\bb,\vv_r,c_r)
$$
for all $r \in [R]$ and therefore $\bomega$ is a global minimum.

For part (ii), define the sets
$$
[R]_{0} := \{ r \in [R] : \vv_r = \0 \} \qquad \text{and} \qquad [R]_{+} := \{r \in [R] : \vv_r \neq \0 \}
$$
as those classes where $\vv_r$ does and does not vanish, respectively. If $r \in [R]_{0}$ then $\vv_r = \0$ and so as a function of $c_r$ the corresponding loss $\Loss^{(r)}$ takes the form
$$
\Loss^{(r)}(W,\bb,\vv_r,c_r) = \frac{\sigma(1 - c_r) + \sigma(1 + c_r)}{2}
$$
due to the equal weight hypothesis. A local minimum $\bomega$ must therefore have $c_r \in [-1,1],$ and as $\Loss^{(r)}$ is constant in for $c_r \in [-1,1]$ it suffices to assume that $c_r \in (-1,1)$ for all $r \in [R]_0$ without loss of generality. If $r \in [R]_{+}$ then $(W,\bb,\vv_r,c_r)$ is a global minimum of $\Loss^{(r)}$ by part (i), and since the $\xx^{(i)}$ are separable a global minimum of $\Loss^{(r)}$ must have zero loss. Thus each of the $N$ equalities
$$
y^{(i,r)}\big( \langle \vv_r , \xx^{(i,1)} \rangle + c_r \big) \geq 1
$$
must hold. By replacing $(\vv_r,c_r)$ with $(\lambda \vv_r , \lambda c_r)$ for any $\lambda > 1$ with $\lambda$ arbitrarily close to $1$ it therefore suffices to assume that 
$$
y^{(i,r)}\big( \langle \vv_r , \xx^{(i,1)} \rangle + c_r \big) > 1
$$
for all $i \in [N], r \in [R]_{+}$ without loss of generality. In other words, by combining the case $r \in [R]_0$ and the case $r \in [R]_+$ it suffices to assume that the local minimizer $\bomega$ obeys
$$
\sign\big( 1 - y^{(i,r)}( \langle \vv_r , \xx^{(i,1)} \rangle + c_r ) \big) \in \{-1,1\}
$$
for all $i \in [N], r \in [R]$ without loss of generality. For such a local minimizer, let $\ww_h , h \in [H]$ denote the rows of $W$. By relabelling the data points if necessary assume that
$$
\langle \ww_h , \xx^{(1)} \rangle \leq \langle \ww_h , \xx^{(2)} \rangle \leq \cdots \leq \langle \ww_h , \xx^{(N-1)} \rangle  \leq \langle \ww_h , \xx^{(N)} \rangle, 
$$
and let $i_1$ respectively denote the greatest index for which the equality
$$
\langle \ww_h , \xx^{(i)}\rangle + \bb_h = 0
$$
holds. Thus $0 = \langle \ww_h , \xx^{(i_1)}\rangle + \bb_h < \langle \ww_h , \xx^{(i_1+1)}\rangle + \bb_h$, and so decreasing $\bb_h$ by any amount smaller than
$$
\langle \ww_h , \xx^{(i_1)} - \xx^{(i_1+1)}\rangle
$$
gives a row/bias pair $(\ww_h,\bb_h)$ for which
$$
\sign\big(  \langle \ww_h , \xx^{(i)}\rangle + \bb_h \big) \in \{-1,1\}
$$
for all $i \in [N]$. Applying such a decrease to all $\bb_h,h \in [H]$ if necessary gives
$$
\sign\big(  W\xx^{(i)} + \bb \big) \in \{-1,1\}^{H}
$$
for all $i \in [N]$. Taking the size of these decreases sufficiently small thus yields a local minimizer $\bomega$ of $\E_{\gamma}$ for which the signature functions obey
$$
\sss^{(i,1)}(W,\bb,\vv_1,c_1,\ldots,\vv_R,c_r) \in \{-1,1\}^{NH} \quad \text{and} \quad \sss^{(i,2)}(W,\bb,\vv_1,c_1,\ldots,\vv_R,c_r) \in \{-1,1\}^{NR} 
$$
for all $i \in [N],$ that is the point $\breve \bomega := (W,\bb,\vv_1,c_1,\ldots,\vv_R,c_R)$ lies in the interior of a cell on which the loss $\Loss(\breve \bomega)$ is smooth. Moreover, the corresponding replicated point $\bomega = (\bomega^{(1)},\ldots,\bomega^{(R)})$ is a local minimum of $\E_{\gamma}$ and, by lemma \ref{lem:locminlem}, of $\Loss$ as well. Thus $\Loss$ attains a local minimum on the interior of a cell. But as $\alpha > 0$ and the $\xx^{(i)}$ are separable, the decomposition of lemma \ref{one_hidden} shows that this can happen only if
$$
\Loss^{(r)}(W,\bb,\vv_r,c_r) = 0
$$
for all $r \in [R],$ and so $\bomega$ is a global minimizer as claimed.
\end{proof}

\end{document}